\pdfoutput=1
\documentclass{article}

\usepackage[main, final]{neurips_2026}

\usepackage[utf8]{inputenc} %
\usepackage[T1]{fontenc}    %
\usepackage{hyperref}       %
\usepackage{url}            %
\usepackage{booktabs}       %
\usepackage{amsfonts}       %
\usepackage{nicefrac}       %
\usepackage{microtype}      %
\usepackage{xcolor}         %

\usepackage{amsmath}
\usepackage{amssymb}
\usepackage{mathtools}
\usepackage{amsthm}
\usepackage{commath}

\usepackage{enumitem}
\usepackage{siunitx}
\usepackage{colortbl}

\mathtoolsset{showonlyrefs}

\theoremstyle{plain}

\newtheorem{proposition}{Proposition}
\newtheorem{lemma}{Lemma}
\newtheorem{corollary}{Corollary}
\newtheorem{definition}{Definition}
\newtheorem{assumption}{Assumption}
\newtheorem{remark}{Remark}

\usepackage{pgfplots}
\pgfplotsset{compat=1.18}
\usepgfplotslibrary{groupplots}
\usetikzlibrary{shapes.geometric}
\usepgfplotslibrary{fillbetween}

\DeclarePairedDelimiterX{\infdivx}[2]{[}{]}{#1\;\delimsize\|\;#2}
\newcommand{\kld}{\mathbb{D}_{\mathrm{KL}}\infdivx*}
\newcommand{\alphad}{\mathbb{D}_{\alpha}\infdivx*}
\newcommand{\timesit}[1]{%
  \text{\fontfamily{ptm}\itshape #1}%
}

\AtBeginDocument{
    \setlength{\abovedisplayskip}{4pt plus 1pt minus 2pt}
    \setlength{\belowdisplayskip}{4pt plus 1pt minus 2pt}
    \setlength{\abovedisplayshortskip}{1pt plus 2pt}
    \setlength{\belowdisplayshortskip}{1pt plus 1pt minus 1pt}
}

\title{Maximin Robust Bayesian Experimental Design}

\author{%
  Hany Abdulsamad \\
  UvA--Bosch Delta Lab \\
  University of Amsterdam \\
  \texttt{h.abdulsamad@uva.nl} \\
  \And
  Sahel Iqbal \\
  Department of Statistics \\
  University of Oxford \\
  \texttt{sahel.iqbal@stats.ox.ac.uk} \\
  \And
  Christian A. Naesseth \\
  UvA--Bosch Delta Lab \\
  University of Amsterdam \\
  \texttt{c.a.naesseth@uva.nl} \\
  \And
  Takuo Matsubara \\
  Business School \\
  University of Sydney \\
  \texttt{takuo.matsubara@sydney.edu.au} \\
  \And
  Adrien Corenflos \\
  Department of Statistics \\
  University of Warwick \\
  \texttt{adrien.corenflos@warwick.ac.uk} \\
}

\begin{document}

\maketitle

\begin{abstract}
We address the brittleness of Bayesian experimental design under model misspecification by formulating the problem as a max--min game between the experimenter and an adversarial nature subject to information-theoretic constraints. We demonstrate that this approach yields a robust objective governed by Sibson's $\alpha$-mutual information~(MI), which identifies the $\alpha$-tilted posterior as the robust belief update and establishes the R\'enyi divergence as the appropriate measure of conditional information gain. To mitigate the bias and variance of nested Monte Carlo estimators needed to estimate Sibson's $\alpha$-MI, we adopt a PAC-Bayes framework to search over stochastic design policies, yielding rigorous high-probability lower bounds on the robust expected information gain that explicitly control finite-sample error.
\end{abstract}

\section{Introduction}

Experimental design is a foundational component of scientific and engineering inquiry, enabling efficient knowledge acquisition across a wide range of domains including physics, epidemiology, neuroscience, biology, and robotics~\citep{melendez2021designing,cook2008optimal,shababo2013bayesian,liepe2013maximizing,schultheis2020receding}. In such settings, each experiment can involve substantial costs related to time, labor, materials, and financial investments, making it imperative to extract as much information as possible from every trial. By explicitly modeling uncertainty over latent parameters and quantifying the informativeness of potential measurements, Bayesian experimental design~\citep{lindley1956measure, chaloner1995bayesian} offers a principled framework for selecting experiments that are maximally informative. To do so, it ties decision-making and Bayesian inference into a single objective aiming to find a design minimizing posterior distribution \emph{uncertainty}, measured in terms of entropy and averaged over the marginal distribution of prospective measurements. This information-theoretic approach enables efficient allocation of experimental effort, reducing both cost and risk while accelerating scientific discovery, particularly where simulators enable large-scale evaluation of designs prior to deployment.

However, accurate simulators remain the exception rather than the rule, even as modeling capabilities continue to improve. Mechanical, chemical, and biological systems often exhibit complex and only partially understood interactions that cannot be faithfully captured by idealized assumptions~\citep{rondelez2012competition, van2024iterative}. Designs optimized under misspecified simulators may therefore lead to ineffective experiments or, in extreme cases, counterproductive outcomes. Hence, when the model is misspecified, robustness is paramount both in the posterior update and in the data generation governed by the marginal distribution. The decision-making literature has a rich history of addressing misspecification in the data-generating process~\citep{whittle1990riskb, petersen2002minimax}, while the Bayesian inference literature has developed extensive methodology for better posterior updates under likelihood misspecification~\citep{grunwald2012safe, bissiri2016general, ghosh2016robust, holmes2017assigning, grunwald2017inconsistency}. The central challenge in robust Bayesian experimental design is to unify these perspectives in a single framework that treats robustness in data generation and inference jointly, capturing their interaction rather than addressing them in isolation.

This problem has attracted increasing attention recently. One popular approach to robustify Bayesian inference is to replace the log-likelihood with a generalized utility yielding what are known as Gibbs posteriors~\citep{zhang2006from, bissiri2016general}. Recent contributions to robust Bayesian experimental design have adopted this framework, substituting standard posteriors appearing in the expected information gain objective with Gibbs posteriors. Additionally, these methods modify the generative component, either by introducing user-specified distributions~\citep{overstall2025gibbs} or by constructing simulators implicitly defined by the chosen scoring rule~\citep{barlas2025robust}. While these approaches offer practical mechanisms for robustifying experimental design, their \emph{ad hoc} formulations lack a coherent optimization perspective of misspecification.

In contrast, we propose a principled framework for robust Bayesian experimental design under model misspecification, formulated as a max--min statistical decision problem~\citep{berger1985statistical}. Misspecification is handled through the lens of distributionally robust optimization~\citep{kuhn2025distributionally} by considering worst-case perturbations of the data-generating process within a Kullback--Leibler neighborhood centered at a nominal model representing the experimenter's best approximation to reality. This formulation yields a tractable solution in the form of a R\'enyi mutual information, specifically Sibson's $\alpha$-mutual information~\citep{sibson1969information, verdu2015alpha, esposito2022sibson}, which we interpret as a robust expected information gain. However, this quantity must typically be estimated using nested Monte Carlo estimators, which yield biased and stochastic evaluations and thereby undermine the applicability of standard deterministic optimization methods. We therefore adopt a PAC-Bayes approach~\citep{mcallester1998some} that optimizes randomized design policies and provide high-probability guarantees on the true robust expected information gain.

This article is organized as follows: Section~\ref{sec:background} reviews Bayesian experimental design; Section~\ref{sec:robust-bed} formulates robust experimental design as a max--min problem, yielding a robust expected information gain based on Sibson's $\alpha$-mutual information; Section~\ref{sec:nested-mc} develops a nested Monte Carlo estimator and analyzes its bias and concentration; Section~\ref{sec:pac-bayes} establishes a PAC-Bayes framework for robust design optimization with finite-sample guarantees; Section~\ref{sec:related-work} discusses related work; Section~\ref{sec:experiments} reports numerical evaluations. Proofs, technical assumptions, and other details are deferred to the appendix.

\section{Background}
\label{sec:background}

Let $\Theta$, $\Xi$ and $\mathcal{X}$ denote the sets of parameters, designs, and outcomes of an experiment, respectively. We write $\mathcal{P}(W)$ for the set of all probability densities defined on a set $W$. In Bayesian experimental design, an experimenter selects a design $\xi \in \Xi$ with the objective of inferring latent parameters $\theta \in \Theta$ from observed measurements $x \in \mathcal{X}$. The experimenter begins with a prior belief $p(\theta)$ over the parameters and assumes a conditional likelihood function $p(x \mid \theta, \xi)$, which specifies how the design $\xi$ relates the latent parameter $\theta$ to the observable outcome $x$. The induced marginal distribution of the measurement is $p(x \mid \xi)$ and the Bayes posterior is $p(\theta \mid x, \xi) \propto p(\theta) \, p(x \mid \theta, \xi)$.

{\color{black}
\paragraph{Expected Information Gain.} Bayesian experimental design admits a decision-theoretic formulation through the framework of \emph{scoring rules}~\citep{gneiting2007strictly}. For a latent parameter $\theta$, the experimenter issues a probabilistic forecast $r \in \mathcal{P}(\Theta)$. The quality of this forecast is evaluated by a scoring rule $S \!: \mathcal{P}(\Theta) \times \Theta \to \mathbb{R}$, where $S(r, \theta)$ denotes the utility of reporting the density $r$ when the true parameter is $\theta$~\citep{bernardo1994bayesian}. A \emph{proper} scoring rule is one that incentivizes the experimenter to report their true belief: if the experimenter's belief is $\omega \in \mathcal{P}(\Theta)$, then the expected utility is maximized by reporting exactly $r = \omega$~\citep{savage1971elicitation, gneiting2007strictly}, so that:
\begin{equation}
    \label{eq:proper-score}
    \omega \in \operatorname*{arg\,max}_{r \in \mathcal{P}(\Theta)} \, \int S(r, \theta) \, \omega(\theta) \dif \theta.
\end{equation}
If the maximizer is unique, the scoring rule is \emph{strictly} proper. Before observing data, the experimenter's forecast is the prior density $p(\theta)$. After observing $x$ under design $\xi$, the experimenter may issue a refined forecast. We represent a data-dependent forecast rule by the map $\psi_{\xi} \!: \mathcal{X} \to \mathcal{P}(\Theta)$. The value of an experiment is the optimal post-observation score improvement over the prior forecast:
\begin{align}
    \label{eq:score-improvement}
    I(\xi) & \coloneq \int p(x \mid \xi) \left[ \sup_{r_x \in \mathcal{P}(\Theta)} \, \int S(r_x, \theta) \, p(\theta \mid x, \xi) \dif \theta \right] \dif x - \int S(p, \theta) \, p(\theta) \dif \theta \\
    & = \sup_{\psi_{\xi}:\mathcal{X}\to\mathcal{P}(\Theta)} \int p(\theta, x \mid \xi) \left[ S(\psi_{\xi}(x), \theta) - S(p, \theta) \right] \dif \theta \dif x.
\end{align}
Only the post-observation forecast rule $\psi_{\xi}(x)$ is subject to optimization here, while the pre-observation forecast $p(\theta)$ is a fixed reference against which the informativeness of $\xi$ is measured. By~\eqref{eq:proper-score}, given a proper scoring rule, the supremum of this objective is attained at $\psi^{\star}_{\xi}(x) = p(\cdot \mid x, \xi)$.

This construction may appear unnecessary at first glance. The fact that the experimenter's optimal forecast is the standard Bayes posterior is intuitive. This formulation, however, becomes advantageous when the joint law $p(\theta, x \mid \xi)$ is misspecified, as we discuss in Section~\ref{sec:robust-bed}.

Under the strictly proper logarithmic scoring rule $S_{\log}(r, \theta) \coloneq \log r(\theta)$, the value $I(\xi)$ recovers Shannon's expected information gain~\citep[EIG;][]{lindley1956measure, bernardo1979expected}, that is ${\int p(\theta, x \mid \xi) \big[ \log p(\theta \mid x, \xi) - \log p(\theta) \big] \dif \theta \dif x}$, the mutual information between $\theta$ and $x$ under $\xi$.
This objective has seen wide adoption in modern Bayesian experimental design~\citep[BED;][]{chaloner1995bayesian, rainforth2024modern}.

\begin{definition}[Expected information gain] \label{def:eig_equations}
    The expected information gain of a design $\xi$, defined as the mutual information between the parameters $\theta$ and outcomes $x$, admits several representations:
    \begin{subequations}
        \begin{align}
            I(\xi) \coloneq I(\theta;x)(\xi)
            & = \mathbb{E}_{p(x \mid \xi)} \left[ \mathbb{H}\left[ p(\theta) \right] - \mathbb{H}\left[ p(\theta \mid x, \xi) \right] \right] \\
            & = \mathbb{E}_{p(x \mid \xi)} \Big[ \kld{p(\theta \mid x, \xi)}{p(\theta)} \Big] \label{eq:eig-param-div} \\
            & = \mathbb{E}_{p(\theta)} \Big[ \kld{p(x \mid \theta, \xi)}{p(x \mid \xi)} \Big] \label{eq:eig-lklhd-div} = I(x; \theta)(\xi)
        \end{align}
    \end{subequations}
    where $\mathbb{H}$ denotes Shannon's differential entropy and $\mathbb{D}_{\mathrm{KL}}$ the Kullback--Leibler divergence.
\end{definition}
}
Although Shannon's mutual information is symmetric, so that $I(\theta; x) \equiv I(x; \theta)$, the two orderings admit distinct interpretations. The ordering $I(\theta; x)$, given by~\eqref{eq:eig-param-div}, reflects the inferential perspective, measuring the expected relative entropy from posterior to
the prior of the latent parameter $\theta$, after observing data $x$. In contrast, the ordering $I(x; \theta)$, as expressed in~\eqref{eq:eig-lklhd-div}, offers a predictive interpretation, quantifying the extent to which knowledge of $\theta$ explains the variability of the outcomes $x$. While these viewpoints coincide numerically for Shannon's mutual information, the distinction is important when considering general information measures for which symmetry need not hold.

Given this definition, the goal of Bayesian experimental design is therefore to select the designs $\xi^{\star}$ that maximize the expected information gain: $\xi^{\star} = \arg\max_{\xi \in \Xi} \, I(\xi)$.

\paragraph{Stochastic Design Policies.} Standard experimental design treats the expected information gain as a function of deterministic designs $\xi \in \Xi$. This approach is adequate provided an exact and deterministic oracle of the objective is available. However, the expected information gain is often intractable and must be approximated by nested Monte Carlo estimators~\citep{rainforth2018nesting}, yielding \emph{biased} and \emph{noisy} oracles.
To address this challenge, we turn to a PAC-Bayes framework, which provides high-probability guarantees on the estimation error but requires modeling designs as random variables~\citep{flynn2023pac, alquier2024pacbayes}.

Rather than committing to a single design $\xi$, we allow the experimenter to specify a probability distribution over the design space, denoted by $\pi \in \Pi$, referred to as a stochastic \emph{policy}. Concretely, given a generative process $p(\theta, x, \xi) = p(\theta) \, p(x \mid \theta, \xi) \, \pi(\xi)$, the expected information gain associated with a policy $\pi$ is defined as the \emph{conditional} mutual information between $x$ and $\theta$ given $\xi$:
\begin{equation}
    \label{eq:info-gain-stochastic-policy}
    I(\pi) \coloneq \kld{p(\theta, x, \xi)}{p(\theta) \, p(x \mid \xi) \, \pi(\xi)} = \mathbb{E}_{\pi} \Big[ \kld{p(\theta, x \mid \xi)}{p(\theta) \, p(x \mid \xi)} \Big].
\end{equation}
The design optimization problem is reframed as $\pi^{\star} = \arg\max_{\pi \in \Pi} \, I(\pi)$. Beyond satisfying PAC-Bayes requirements, this relaxation allows us to leverage a rich literature of stochastic policy optimization techniques~\citep{peters2010relative, haarnoja2018soft}, which have shown strong performance in amortizing sequential Bayesian experimental design~\citep{blau2022optimizing, iqbal2024nesting}.

\section{Robust Experimental Design}
\label{sec:robust-bed}

When the measurements encountered at deployment are indeed generated from the assumed conditional likelihood $p(x \mid \theta^{\star}, \xi)$ for some $\theta^{\star} \in \Theta$, \citet{bernardo1979expected} provides a decision-theoretic justification for the expected information gain. In this well-specified regime, the EIG faithfully reflects the anticipated value of an experiment.

However, in many applications the likelihood is only a modeling surrogate and may differ substantially from the true data-generating process. Such discrepancies are a source of systematic error in statistical inference under model misspecification~\citep{grunwald2017inconsistency}. The expected information gain is particularly fragile because of its nested dependence on the likelihood: it determines both the posterior $p(\theta \mid x, \xi) \propto p(\theta) \, p(x \mid \theta, \xi)$ and the marginal $p(x \mid \xi) = \int p(x \mid \theta, \xi) \, p(\theta) \dif \theta$ that defines the outer expectation. Given misspecification, $I(\xi)$ may become an unreliable utility, favoring designs that appear highly informative under the assumed model but fail to deliver the anticipated information gain when evaluated against the true data-generating process.

{\color{black}
\paragraph{Maximin Experimental Design.} To treat this issue systematically, we cast experimental design under misspecification as a two-player zero-sum Stackelberg game between the experimenter and \emph{nature} \citep{vonNeumann1994theory, bacsar1998dynamic, grunwald2004game}. Nature selects a joint law $q \!: \Xi \to \mathcal{P}(\Theta \times \mathcal{X})$ from an admissible set $\mathcal{Q}$ as the data-generating process and the experimenter counters with a design policy $\pi \in \Pi$. By perturbing the full joint $q(\theta, x \mid \xi)$, nature is allowed to represent scenarios of both prior miscalibration and likelihood misspecification.

We assume the experimenter knows the admissible set $\mathcal{Q}$ but not nature's specific choice $q \in \mathcal{Q}$, whereas nature observes the policy $\pi \in \Pi$ but not the experimenter's realized design draw $\xi$. The experimenter evaluates a policy via its expected worst-case payoff $V(\pi, q) \coloneq \mathbb{E}_{\pi} \left[ U(\xi, q) \right]$, and seeks robustness against nature's choice by solving the maximin problem~\citep{wald1950statistical, berger1985statistical}
\begin{equation}
    \label{eq:game-value}
    \pi^{\star} \coloneq \arg\sup_{\pi \in \Pi} \, \inf_{q \in \mathcal{Q}} \, V(\pi, q).
\end{equation}
To define the pointwise payoff $U(\xi, q)$, we return to the decision-theoretic formulation of Section~\ref{sec:background}. Assuming the experimenter commits to the prior $p(\theta)$ as a fixed reference forecast, the value of a design $\xi$ against an adversarial joint law $q(\theta, x \mid \xi)$ can be expressed as:
\begin{equation}
    \label{eq:pointwise-utility}
    U(\xi, q) \coloneq \sup_{\psi_{\xi}:\mathcal{X}\to\mathcal{P}(\Theta)} \int q(\theta, x \mid \xi) \Big[ \log \psi_{\xi}(x)(\theta) - \log p(\theta) \Big] \dif \theta \dif x.
\end{equation}
For any fixed $q(\theta, x \mid \xi)$, the supremum is attained at $\psi_{\xi}^{\star}(x) = q(\cdot \mid x, \xi)$ and the payoff reduces to:
\begin{equation}
    \label{eq:pointwise-utility-kl}
    U(\xi, q) = \kld{q(\theta, x \mid \xi)}{p(\theta) \, q(x \mid \xi)}.
\end{equation}
This quantity is not, in general, the standard mutual information under $q$, but rather a divergence from $q(\theta, x \mid \xi)$ to a mixed law combining the experimenter's prior $p(\theta)$ with nature's marginal $q(x \mid \xi)$.
}

When the set $\mathcal{Q}$ is unrestricted, the worst-case formulation is vacuous, since nature can trivially drive the information gain to zero by enforcing independence between $\theta$ and $x$. A meaningful notion of robustness therefore requires $\mathcal{Q}$ to constrain nature's flexibility by specifying how far $q(\theta, x \mid \xi)$ may deviate from the nominal reference $p(\theta, x \mid \xi)$, which encodes the experimenter's best guess of the true data-generating process. Because the experimenter commits to a stochastic policy, it is natural to specify this restriction with an average misspecification budget over the support of $\pi$.

Following the framework of distributionally robust optimization~\citep{kuhn2025distributionally}, we formalize this restriction by introducing a Kullback--Leibler ambiguity set $\mathcal{Q}_{\rho}$.

\begin{definition}[Ambiguity set]
    \label{def:ambiguity-set}
    Let $\rho > 0$ be a misspecification budget. The set $\mathcal{Q}_{\rho}(\pi)$ contains all distributions $q: \Xi \to \mathcal{P}(\Theta \times \mathcal{X})$ within an average Kullback--Leibler neighborhood of $p(\theta, x \mid \xi)$:
    \begin{equation}
        \label{eq:ambiguity-set}
        \mathcal{Q}_{\rho}(\pi) \coloneq \left\{ q(\theta, x \mid \xi) \, \Big| \, \mathbb{E}_{\pi} \left[ \kld{q(\theta, x \mid \xi)}{p(\theta, x \mid \xi)} \right] \le \rho \right\},
    \end{equation}
    where $p(\theta, x \mid \xi)$ denotes the experimenter's nominal model.
\end{definition}
This global constraint allows nature to strategically allocate its budget under a bounded average distortion, attacking informative design regions more aggressively while ignoring uninformative ones.

\paragraph{Robust Expected Information Gain.} We study the inner minimization problem $\inf_{q \in \mathcal{Q}_{\rho}} V(\pi, q)$, which defines the worst-case utility from the experimenter's perspective. By Lagrangian dualization~\citep{boyd2004convex}, the global minimization decomposes into pointwise regularized problems for every design $\xi \in \Xi$, coupled through a shared misspecification budget.

\begin{lemma}[Minimization decomposition]
    \label{lem:minimization-decomposition}
    The worst-case criterion reduces to a dual form $\inf_{q \in \mathcal{Q}_{\rho}} V(\pi, q) = \sup_{\beta > 0} \big\{ \mathbb{E}_{\pi} \left[ \mathcal{J}_{\beta}(\xi)\right]  - \beta \rho \big\}$, where $\mathcal{J}_{\beta}(\xi)$ is the regularized utility:
    \begin{equation}
        \label{eq:regularized-pointwise-utility}
        \mathcal{J}_{\beta}(\xi) \coloneq \inf_{q} \, \left\{ U(\xi, q) + \beta \, \kld{q(\theta, x \mid \xi)}{p(\theta, x \mid \xi)} \right\}.
    \end{equation}
\end{lemma}
The radius $\rho$, or equivalently its dual parameter $\beta$, is a robustness knob. Smaller $\rho$ reflects stronger trust in the nominal model $p(\theta, x \mid \xi)$, while larger $\rho$ permits more perturbations. The next result gives the corresponding pointwise objective in closed form.
\begin{proposition}[Robust expected information gain]
    \label{prop:robust-expected-info-gain}
    For any design $\xi \in \Xi$ and regularization parameter $\beta > 0$, let $\alpha \coloneq \beta / (1 + \beta) \in (0,1)$. The utility $\mathcal{J}_{\beta}(\xi)$ admits the closed-form expression:
    \begin{equation}
        \!\!\! \mathcal{J}_{\beta}(\xi) \! = \! \inf_{\nu} \, \alphad{p(\theta, x \mid \xi)}{p(\theta) \, \nu(x \mid \xi)} \! = \! \alphad{p(\theta, x \mid \xi)}{p(\theta) \, p_{\alpha}(x \mid \xi)} \coloneq I^{\timesit{S}}_{\alpha}(\theta; x)(\xi), \label{eq:robust-expected-info-gain}
    \end{equation}
    where $\mathbb{D}_{\alpha}$ denotes R\'enyi's divergence of order $\alpha$ and the tilted marginal $p_{\alpha}(x \mid \xi)$ is given by:
    \begin{equation}
        \label{eq:alpha-tilted-marginal}
        p_{\alpha}(x \mid \xi) \propto \left[ \mathbb{E}_{p(\theta)} \Big[ p(x \mid \theta, \xi)^{\alpha} \Big] \right]^{1/\alpha}.
    \end{equation}
\end{proposition}
The quantity $I^{\timesit{S}}_{\alpha}(\theta; x)$ is Sibson's $\alpha$-mutual information~\citep{sibson1969information, verdu2015alpha, esposito2022sibson, esposito2024sibson}, interpreted as a \emph{robust expected information gain} with a misspecification order $\alpha$.

\begin{remark}
    Unlike Shannon's mutual information, Sibson's $\alpha$-mutual information is not a symmetric measure of information, $I^{\timesit{S}}_{\alpha}(\theta; x)(\xi) \neq I^{\timesit{S}}_{\alpha}(x; \theta)(\xi)$~\citep{esposito2024sibson}.
\end{remark}
\begin{corollary}[Worst-case generative process]
    \label{cor:worst-case-process}
    Let $\alpha \in (0, 1)$ be the misspecification order determined by the ambiguity radius $\rho$. The worst-case joint distribution $q^{\star}(\theta, x \mid \xi)$ that minimizes the objective in Proposition~\ref{prop:robust-expected-info-gain} is given by the geometric mixture:
    \begin{equation}
        q^{\star}(\theta, x \mid \xi) \propto \Big[ p(\theta) \, p_{\alpha}(x \mid \xi) \Big]^{1 - \alpha} \Big[ p(\theta, x \mid \xi) \Big]^{\alpha},
    \end{equation}
    where the $\alpha$-marginal $p_{\alpha}(x \mid \xi)$ is defined by \eqref{eq:alpha-tilted-marginal}.
    Consequently, the worst-case posterior is
    \begin{equation}
        \label{eq:titled-posterior}
        q^{\star}(\theta \mid x, \xi) \propto p(\theta) \, p(x \mid \theta, \xi)^{\alpha}.
    \end{equation}
\end{corollary}
Corollary~\ref{cor:worst-case-process} shows how Sibson's $\alpha$-mutual information leads to robustness at the Bayesian inference level. The $\alpha$-tilted posterior~\eqref{eq:titled-posterior} arises as the optimal belief update under the worst-case generative process. Therefore, acting consistently with this adversarial assumption requires the experimenter to update beliefs about $\theta$ using this tempered rule rather than the standard Bayesian posterior. By down-weighting the misspecified likelihood with $\alpha \in (0, 1)$, this result aligns with generalized Bayesian inference techniques for handling model misspecification~\citep[see, e.g.,][and references therein]{grunwald2012safe, watson2016approximate, grunwald2017inconsistency, knoblauch2022optimization}.

At the design level, robustness is governed by the monotonicity of Sibson's $\alpha$-mutual information in $\alpha$~\citep{esposito2024sibson}. As $\alpha \to 1$, the ambiguity set contracts, adversarial influence vanishes, and $I^{\timesit{S}}_{\alpha}(\theta;x)$ recovers the standard expected information gain in Definition~\ref{def:eig_equations}. As $\alpha$ decreases the objective becomes more conservative, favoring designs whose information gain is stable under misspecification. In the limit $\alpha \to 0$, nature exerts maximal influence, nullifying any gain.
\begin{proposition}[Uninformative design]
    \label{prop:uninformative}
    Under Assumption~\ref{asm:finite-surprise}, as $\alpha \to 0$, the robust expected information gain converges to $0$:
    \begin{equation}
        \lim_{\alpha \to 0} \, \sup_{\xi \in \Xi} I^{\timesit{S}}_{\alpha}(\theta; x)(\xi) = \lim_{\alpha \to 0} \, \sup_{\xi \in \Xi} \, \alphad{p(\theta, x \mid \xi)}{p(\theta) \, p_{\alpha}(x \mid \xi)} = 0.
    \end{equation}
\end{proposition}

{\color{black}
The choice of $\alpha$ is therefore a prior modeling choice, much like the choice of $p(\theta)$. Before observing data from the true data-generating process, the experimenter must specify a nominal model $p(\theta, x \mid \xi)$, based on domain knowledge, and then decide how much trust to place in that model. Values $\alpha \approx 1$ express high confidence and recover behavior close to the standard well-specified setting, while values $\alpha \ll 1$ reflect substantial model uncertainty and favor designs that remain informative under stronger adversarial misspecification. Thus, $\alpha$ should be interpreted as a \emph{subjective} modeling variable that controls the experimenter's assumptions about the set of plausible data-generating processes.
}

\paragraph{Risk-Sensitive Interpretation.} Another interpretation of the robust measure $I^{\timesit{S}}_{\alpha}(\theta;x)(\xi)$ from Proposition~\ref{prop:robust-expected-info-gain} is given by the following decomposition.
{\color{black}
\begin{proposition}[Risk-sensitive form]
    \label{prop:robust-conditional-info-gain}
    The robust objective admits the risk-sensitive representation
    \begin{equation}
    \label{eq:risk-sensitive-form}
        I^{\timesit{S}}_{\alpha}(\theta;x)(\xi) = \frac{\alpha}{\alpha - 1} \log \mathbb{E}_{p(x \mid \xi)} \left[ \exp \left\{ \frac{\alpha - 1}{\alpha} G_{\alpha}(x, \xi) \right\} \right],
    \end{equation}
    where $G_{\alpha}(x, \xi) \coloneq \alphad{p(\theta \mid x, \xi)}{p(\theta)}$ is the robust conditional information gain for a specific outcome-design pair $(x, \xi)$. Additionally, for $\alpha$ in a neighborhood of $1$, this representation admits the following second-order expansion
    \begin{equation}
        \label{eq:risk-sensitive-expansion}
        I^{\timesit{S}}_{\alpha}(\theta; x)(\xi) \approx I(\xi) + \frac{\alpha - 1}{2} \mathbb{V}_{p(\theta, x \mid \xi)} \left[ \log \frac{p(\theta \mid x, \xi)}{p(\theta)} \right].
    \end{equation}
\end{proposition}
The form in~\eqref{eq:risk-sensitive-form} directly mirrors the interpretation of Shannon's mutual information, defined in~\eqref{eq:eig-param-div}, as an expected divergence between prior and posterior in the well-specified setting.
However, it identifies the conditional utility $G_{\alpha}(x, \xi)$ as the R\'enyi divergence between the nominal posterior and the prior, while the overall robust expected information gain $I_{\alpha}^{\timesit{S}}(\theta; x)(\xi)$ aggregates these conditional gains through a risk-sensitive exponential average.

Moreover, the expansion in~\eqref{eq:risk-sensitive-expansion} makes risk-sensitivity explicit. Approximately, $I_{\alpha}^{\timesit{S}}(\theta; x)(\xi)$ equals Shannon's expected information gain $I(\xi)$ with an additional variance term that, for $\alpha \in (0, 1)$, penalizes designs whose realized information gain fluctuates across design-outcome pairs.
}
\section{Nested Monte Carlo Estimator}
\label{sec:nested-mc}

The robust expected information gain $I_{\alpha}^{\timesit{S}} \coloneq I_{\alpha}^{\timesit{S}}(\theta;x)$ only has closed-form solutions in special cases. We propose to use a simple nested Monte Carlo estimator to estimate it in the general case. We use the following decomposition of Sibson's $\alpha$-mutual information for estimation.

\begin{corollary}[Nested representation of $I_{\alpha}^{\timesit{S}}$]
    \label{cor:nested-representation}
    The robust expected information gain $I^{\timesit{S}}_{\alpha}(\xi)$ can be expressed in terms of nested expectations over nonlinear maps:
    \begin{align}
        \label{eq:sibson-nested}
        I^{\timesit{S}}_{\alpha}(\xi)
            & = \frac{\alpha}{\alpha-1}\log \mathbb{E}_{p(x \mid \xi)} \! \left[ \left( \mathbb{E}_{p(\theta)} \! \left[w(x, \theta, \xi)^{\alpha} \right] \right)^{1/\alpha} \right], \,\, \text{with} \,\, w(x, \theta, \xi) \coloneq p(x \mid \theta, \xi) / p(x \mid \xi).
    \end{align}
\end{corollary}

To simplify the exposition, we define the outer logarithmic transformation $f(y) \coloneq \alpha / (\alpha - 1) \log(y)$ and the inner power function $h(u) \coloneq u^{1/\alpha}$. Using this notation, the robust objective can be written compactly as $I^{\timesit{S}}_{\alpha}(\xi) = f(g(\xi))$, where $g(\xi) \coloneq \mathbb{E}_{p(x \mid \xi)} \Big[ h\big( \mathbb{E}_{p(\theta)} \left[w(x, \theta, \xi)^{\alpha} \right] \big) \Big]$.

\begin{definition}[Nested Monte Carlo estimator $\tilde{I}_{\alpha}^{\timesit{S}}$]
    \label{def:robust-eig-estimator}
    For a design $\xi$ and misspecification order $\alpha$, the empirical estimator $\tilde{I}^{\timesit{S}}_{\alpha}(\xi) \coloneq f(\tilde{g}(\xi))$ is constructed using a nested Monte Carlo scheme. Let $\{(\theta^{(i)}, x^{(i)})\}_{i=1}^{N}$ be a set of samples drawn from the joint distribution, and $\{\theta^{(i,j)}\}_{j=1}^{M}$ an independent set of samples drawn from the prior. We define $\tilde{g}(\xi) \coloneq \frac{1}{N} \sum_{i=1}^{N} h\big( \frac{1}{M} \sum_{j=1}^{M} w(x^{(i)}, \theta^{(i, j)}, \xi)^{\alpha} \big)$. When the density ratio is intractable, we rely on a contrastive estimator $\tilde{w}$ using $K$ auxiliary samples: $\tilde{w}(x^{(i)}, \theta^{(i,j)}, \xi) \coloneq p(x^{(i)} \mid \theta^{(i,j)}, \xi) / \frac{1}{K} \sum_{k=1}^{K} p(x^{(i)} \mid \theta^{(i,k)}, \xi)$.
\end{definition}

Appendix~\ref{app:nested-estimator} provides a numerically stable method for computing the estimator.
Because the estimator applies the nonlinear functions to empirical averages rather than exact expectations, Jensen's inequality implies that it is biased for any finite sample sizes $N$ and $M$. Nonetheless, the analysis of \citet{rainforth2018nesting} establishes that it remains consistent and its mean squared error decays as $\mathcal{O}(1/N + 1/M)$, when $w$ is available in closed form.

Next, we provide explicit bounds for its bias under Assumption~\ref{asm:robust-eig-regularity}, which defines regularity constants $L_f$, $L_h$, $C_h$, $\sigma_h$, and $\sigma_w$ for the functions $f$, $h$, and $w$.

\begin{proposition}[Bias bound for $\tilde{I}^{\timesit{S}}_{\alpha}$]
    \label{prop:robust-eig-bias}
    Let $\tilde{I}_{\alpha}^{\timesit{S}}(\xi)$ be the estimator of $I_{\alpha}^{\timesit{S}}(\xi)$ given by Definition~\ref{def:robust-eig-estimator}. Under Assumption~\ref{asm:robust-eig-regularity}, for any design $\xi \in \Xi$, the bias of $\tilde{I}^{\timesit{S}}_{\alpha}(\xi)$ is bounded by:
    \begin{equation}
        \Big| \mathbb{E} \big[ \tilde{I}^{\timesit{S}}_{\alpha}(\xi) \big] - I^{\timesit{S}}_{\alpha}(\xi) \Big| \! \leq \! L_{f} \Bigg( \frac{L_{h} \sigma_{w}}{\sqrt{M}} + \sqrt{\frac{\sigma_{h}^{2}}{N} \! + \! \frac{4 C_{h} L_{h} \sigma_{w}}{N\sqrt{M}}} \Bigg).
        \vspace{-\belowdisplayshortskip}
    \end{equation}
\end{proposition}

\section{PAC-Bayesian Design Policies}
\label{sec:pac-bayes}

In the previous section, we introduced a nested Monte Carlo estimator $\tilde{I}^{\timesit{S}}_{\alpha}(\xi)$ for the robust expected information gain $I^{\timesit{S}}_{\alpha}(\xi)$. However, this estimator is a biased and noisy oracle, introducing errors that can mislead design selection if optimized naively. To address this, we adopt stochastic design policies and control the estimator error through a PAC-Bayesian approach. In Appendix~\ref{app:pac-bayes-advantage}, we provide a discussion of classical PAC analysis and the PAC-Bayesian approach.

We view design selection as a PAC-Bayes continuous bandit problem~\citep{flynn2023pac}, where the designs $\xi$ are arms and the true rewards are determined by $I_{\alpha}^{\timesit{S}}$, observed only through the noisy surrogate $\tilde{I}_{\alpha}^{\timesit{S}}$. The PAC-Bayesian formulation optimizes a high-probability lower bound on the true robust objective by penalizing the divergence to a prior policy $\pi_{0}$, thereby balancing empirical performance against estimator uncertainty~\citep{alquier2024pacbayes}.

To establish the PAC-Bayesian guarantee, we leverage results of prior sections to derive a probability bound on the bias of the robust expected information gain estimator.

\begin{proposition}[Uniform convergence of $\tilde{I}_{\alpha}^{\timesit{S}}$]
    \label{prop:estimator-convergence}
    Let $\tilde{I}_{\alpha}^{\timesit{S}}(\xi)$ be the estimator of $I_{\alpha}^{\timesit{S}}(\xi)$ from Definition~\ref{def:robust-eig-estimator}. Under Assumption~\ref{asm:robust-eig-regularity}, for any tolerance $t > 0$, and for a sufficiently large inner sample size $M \geq 4 L_{f}^{2} L_{h}^{2} \sigma_{w}^{2} / t^{2}$, the probability that the estimator deviates from the oracle by more than $t$ decays exponentially with the outer sample size $N$:
    \begin{equation}
        \mathbb{P} \left( \big| \tilde{I}_{\alpha}^{\timesit{S}}(\xi) - I_{\alpha}^{\timesit{S}}(\xi) \big| \geq t \right) \le 2 \exp\left\{ - N t^{2} / (2 L_{f}^{2} C_{h}^{2}) \right\}.
    \end{equation}
\end{proposition}
Proposition~\ref{prop:estimator-convergence} effectively states that the absolute error of our naive estimator $\tilde{I}^{\timesit{S}}_{\alpha}$ behaves like a sub-Gaussian random variable \emph{if} we pay the cost $M$ to make it so. With this concentration behavior in mind, we can now establish a PAC-Bayes lower bound for the robust expected information gain.

\begin{proposition}[PAC-Bayes lower bound for $I_{\alpha}^{\timesit{S}}$]
    \label{prop:pac-bayes-bound}
    Let $\pi_{0} \in \Pi$ be a prior design policy, $\delta \in (0, 1)$ a confidence level, and $\lambda > 0$ a precision parameter. Provided $M \geq 2 N L_{h}^{2} \sigma_{w}^{2} / (C_{h}^{2} \log (2 / \delta))$, then with probability at least $1 - \delta$, this bound holds for all $\pi \in \Pi$ simultaneously:
    \begin{equation}
        \mathbb{E}_{\pi} \big[I_{\alpha}^{\timesit{S}}(\xi) \big] \ge \mathbb{E}_{\pi} \big[ \tilde{I}_{\alpha}^{\timesit{S}}(\xi) \big] - \frac{\lambda L_{f}^{2} C_{h}^{2}}{4N} - \frac{\kld{\pi}{\pi_{0}} + 2 \log(2/\delta)}{\lambda}.
    \end{equation}
\end{proposition}
Proposition~\ref{prop:pac-bayes-bound} yields a high-probability lower bound on the true \emph{robust} objective expressed in terms of the Monte Carlo surrogate and a KL penalty. Since the remaining terms in the bound are independent of $\pi$, maximizing this lower bound over $\Pi$ reduces to the following variational optimization problem:
\begin{equation}
    \label{eq:design-policy-opt}
    \pi^{\star} \coloneq \arg\max_{\pi \in \Pi} \, \left\{ \mathbb{E}_{\pi} \big[ \tilde{I}_{\alpha}^{\timesit{S}}(\xi) \big] - \frac{\kld{\pi}{\pi_{0}}}{\lambda} \right\},
\end{equation}
whose maximizer is given by a Gibbs policy~\citep{alquier2024pacbayes}: $\pi^{\star}(\xi) \propto \pi_{0}(\xi) \, \exp \big\{ \lambda \, \tilde{I}_{\alpha}^{\timesit{S}}(\xi) \big\}.$ This objective can be optimized using stochastic policy search methods, including natural gradient and mirror descent algorithms~\citep{amari1998natural, kakade2001natural, beck2003mirror, peters2010relative}.
{\color{black}
The role of the precision parameter $\lambda$ is discussed in Appendix~\ref{app:pac-bayes-precision}.
}

\section{Related Work}
\label{sec:related-work}

The two works most closely related to ours methodologically are \citet{go2022robust} and \citet{waite2022minimax}. \citet{go2022robust} also adopt a distributionally robust optimization framework, but focus on prior misspecification, leading to a risk-sensitive objective based on log-exponential averages of Kullback--Leibler divergences rather than a R\'enyi-type divergence. By contrast, our maximin formulation naturally recovers Sibson's $\alpha$-mutual information. \citet{waite2022minimax} study maximin experimental design from a \emph{frequentist} perspective and consider stochastic policies. However, the absence of prior beliefs over $\theta$ leads to a different mechanism for constructing robust designs, and their theoretical analysis is confined to linear models.

Two recent works applied generalized posteriors to Bayesian experimental design: \citet{overstall2025gibbs} and \citet{barlas2025robust}. In both cases, a Gibbs posterior is introduced as an \emph{ad hoc} device to robustify the inference procedure given a design and a measured outcome. In contrast, in our framework the tilted posterior~\eqref{eq:titled-posterior} emerges naturally from the maximin robustness principle as the experimenter's robust belief update rule. This construction leads to a unified formulation in which misspecification is addressed jointly at the level of posterior inference and the data-generating process. In Appendix~\ref{app:compare-barlas}, we discuss the connection to \citet{barlas2025robust} in further detail through two concrete experimental design problems, and show how the resulting objectives differ in general.

A related line of work develops alternative utility functions that incorporate information from the \emph{true, but unknown} data-generating process. \citet{catanach2023metrics} introduce the expected generalized information gain, replacing the KL divergence in~\eqref{eq:eig-param-div} with a discrepancy measure defined under the true process. Since this quantity is not directly computable, their approach requires specifying a model class containing the true process and evaluating the sensitivity of the utility with respect to parameters indexing this class. Similarly, \citet{tang2025generalization} augment the EIG with a penalty on covariate shifts between the training data collected using optimized designs and a reference test set. Therefore, the proposed acquisition function is constrained by the availability of such test data.

Finally, some complementary approaches address misspecification by augmenting the model with expressive components such as Gaussian processes~\citep{feng2015optimal, forster2025improving}.

\section{Numerical Evaluation}
\label{sec:experiments}

Our evaluation focuses on two aspects. First, we highlight the properties of Sibson's $\alpha$-MI as a criterion for robust experimental design, emphasizing that robustness manifests at the levels of the posterior update and design selection. Second, when the robust EIG is only available through a biased and noisy empirical estimator, we show that naive deterministic design optimizers can yield suboptimal designs, hence validating the need for the PAC-Bayes policy for design optimization.

To study the properties of Sibson's $\alpha$-MI, we focus on problem settings in which relevant quantities admit closed-form expressions. This lets us isolate the behavior of Sibson's $\alpha$-MI without confounding numerical artifacts. Specifically, we consider a continuous linear regression problem and a discrete A/B testing problem, which allow for exact computation of posterior updates, conditional information gains, and Sibson's $\alpha$-MI. Details on computations are provided in Appendices~\ref{app:linear-regression} and~\ref{app:ab-testing}:

\begin{enumerate}[leftmargin=2em]
    \item In linear regression, the experimenter assumes that measurements are real-valued responses $x \in \mathbb{R}$ generated according to the Gaussian $p(x \mid \theta, \xi) = \mathrm{N}(x;\xi \, \theta_{1} + \theta_{2}, \ \sigma^2)$, where $\xi \in [-1, 1]$ is the design and the parameters are $\theta = (\theta_{1},\theta_{2})^{\top} \in \mathbb{R}^2$, with prior $p(\theta) = \mathrm{N}(\theta; \mu_{0}, \Sigma_{0})$. The true data-generating process draws from a Student-$t$ distribution $\mathrm{T}(x; \xi \, \theta_{1} + \theta_{2}, \sigma^2, \nu)$ with $\nu$ degrees of freedom and the same conditional mean, with heavier tails than the assumed Gaussian.

    \item In A/B testing, the experimenter assumes measurements are counts $x = (x_a, x_b) \in \{0,\dots,n_a\} \times \{0,\dots,n_b\}$ from a Binomial $p(x \mid \theta, \xi) = \textstyle \prod_{k} \mathrm{Bin}(x_{k}; n_{k}, \theta_{k})$, where $k \in \{a, b\}$. The design $\xi = n_{a} \in \{0,\dots,N_x\}$ controls group allocation~(with $n_{b} = N_{x} - n_{a}$), and the parameters are $\theta = (\theta_{a}, \theta_{b})^{\top} \in [0,1]^2$ with independent Beta priors $p(\theta) = \prod_{k\in\{a, b\}} \mathrm{Beta}(\theta_k; \delta_{k}, \gamma_{k})$.
    The true data-generating process draws $\theta$ from the Beta prior, then for each variant samples a latent rate $r_k \sim \mathrm{Beta}(\kappa \, \theta_k, \kappa \, (1 - \theta_k))$ and measurements $x_k \sim \text{Bin}(n_k, r_k)$, yielding overdispersed Beta-Binomial measurements.
\end{enumerate}
\begin{figure}[t]
    \centering
    \begin{tikzpicture}
    \begin{groupplot}[
        group style={
            group name=infogain,
            group size=1 by 2,
            vertical sep=0.75cm,
        },
        width=5.75cm,
        height=2.92cm,
        ylabel={Frequency},
        xlabel={Information Gain},
        ymin=0,
        scaled y ticks=base 10:-3,
        ytick scale label code/.code={$\times 10^{3}$},
        y tick scale label style={
            at={(axis description cs:0.02,1.02)},
            anchor=south west,
            xshift=-7pt,
            yshift=0pt,
            font=\scriptsize,
        },
        label style={font=\scriptsize},
        tick label style={font=\scriptsize},
        title style={font=\scriptsize},
        grid=major,
        grid style={gray!30},
        xlabel style={yshift=2pt},
        ylabel style={yshift=-5pt},
        title style={yshift=-5pt},
    ]

    \nextgroupplot[
        xlabel={},
        xmin=-1, xmax=10.3,
        xtick={0,2,4,6,8,10},
        title={Linear Regression},
        legend style={
            at={(0.99,0.97)},
            anchor=north east,
            draw,
            font=\tiny,
            /pgfplots/legend cell align=left,
            inner sep=1.pt,
            outer sep=0pt,
            legend image post style={scale=1.0},
            text depth=0pt,
            row sep=-1pt,
        },
    ]
    \addplot[ybar interval, fill=gray!40, draw=black, line width=0.25pt, fill opacity=0.7, area legend]
        table[x=bin_center, y=nominal_count, col sep=comma] {figures/data/linreg_gain_studentt.csv};
    \addlegendentry{Nominal}
    \addplot[ybar interval, fill=gray!150, draw=black, line width=0.25pt, fill opacity=0.7, area legend]
        table[x=bin_center, y=robust_count, col sep=comma] {figures/data/linreg_gain_studentt.csv};
    \addlegendentry{Robust}

    \nextgroupplot[
        xmin=-0.5, xmax=5,
        xtick={0,1,2,3,4,5},
        title={A/B Testing},
    ]
    \addplot[ybar interval, fill=gray!40, draw=black, line width=0.25pt, fill opacity=0.7, area legend]
        table[x=bin_center, y=nominal_count, col sep=comma] {figures/data/abtesting_gain_mixture.csv};
    \addplot[ybar interval, fill=gray!150, draw=black, line width=0.25pt, fill opacity=0.7, area legend]
        table[x=bin_center, y=robust_count, col sep=comma] {figures/data/abtesting_gain_mixture.csv};

    \end{groupplot}

    \begin{scope}[shift={($(infogain c1r1.north east) + (1.05cm,0)$)}]
    \begin{groupplot}[
        group style={
            group name=coverage,
            group size=2 by 1,
            horizontal sep=0.35cm,
            ylabels at=edge left,
            yticklabels at=edge left,
        },
        width=5.0cm,
        height=5.0cm,
        xlabel={Expected Coverage},
        ylabel={Actual Coverage},
        xmin=-0.05, xmax=1.05,
        ymin=-0.05, ymax=1.05,
        label style={font=\scriptsize},
        tick label style={font=\scriptsize},
        title style={font=\scriptsize},
        grid=major,
        grid style={gray!30},
        minor tick num=4,
        major tick length=5pt,
        minor tick length=2pt,
        xlabel style={yshift=2pt},
        ylabel style={yshift=-5pt},
        title style={yshift=-5pt},
        at={(0,0)},
        anchor=north west,
    ]

    \nextgroupplot[
        title={Linear Regression},
        legend style={
            at={(0.03,0.97)},
            anchor=north west,
            draw,
            font=\tiny,
            /pgfplots/legend cell align=left,
            inner sep=1.0pt,
            outer sep=0pt,
            legend image post style={scale=0.5},
            text depth=0pt,
            row sep=-3pt,
        },
    ]
        \addplot[black, thick, loosely dotted, forget plot] coordinates {(0,0) (1,1)};
        \addplot[gray!40, very thick, forget plot] table[x=coverage_levels, y={alpha_1.0000}, col sep=comma] {figures/data/linreg_coverage_studentt.csv};
        \addplot[gray!70, very thick, forget plot] table[x=coverage_levels, y={alpha_0.4870}, col sep=comma] {figures/data/linreg_coverage_studentt.csv};
        \addplot[gray!100, very thick, forget plot] table[x=coverage_levels, y={alpha_0.2371}, col sep=comma] {figures/data/linreg_coverage_studentt.csv};
        \addplot[gray!130, very thick, forget plot] table[x=coverage_levels, y={alpha_0.1155}, col sep=comma] {figures/data/linreg_coverage_studentt.csv};
        \addplot[gray!150, very thick, forget plot] table[x=coverage_levels, y={alpha_0.0562}, col sep=comma] {figures/data/linreg_coverage_studentt.csv};
        \addlegendimage{area legend, fill=gray!40, draw=black, line width=0.25pt}
        \addlegendentry{$\alpha=1.000$}
        \addlegendimage{area legend, fill=gray!70, draw=black, line width=0.25pt}
        \addlegendentry{$\alpha=0.487$}
        \addlegendimage{area legend, fill=gray!100, draw=black, line width=0.25pt}
        \addlegendentry{$\alpha=0.237$}
        \addlegendimage{area legend, fill=gray!130, draw=black, line width=0.25pt}
        \addlegendentry{$\alpha=0.115$}
        \addlegendimage{area legend, fill=gray!150, draw=black, line width=0.25pt}
        \addlegendentry{$\alpha=0.056$}

    \nextgroupplot[title={A/B Testing}]
        \addplot[black, thick, loosely dotted, forget plot] coordinates {(0,0) (1,1)};
        \addplot[gray!40, very thick] table[x=expected_levels, y={alpha_1.0000}, col sep=comma] {figures/data/abtesting_coverage_mixture.csv};
        \addplot[gray!70, very thick] table[x=expected_levels, y={alpha_0.4870}, col sep=comma] {figures/data/abtesting_coverage_mixture.csv};
        \addplot[gray!100, very thick] table[x=expected_levels, y={alpha_0.2371}, col sep=comma] {figures/data/abtesting_coverage_mixture.csv};
        \addplot[gray!130, very thick] table[x=expected_levels, y={alpha_0.1155}, col sep=comma] {figures/data/abtesting_coverage_mixture.csv};
        \addplot[gray!150, very thick] table[x=expected_levels, y={alpha_0.0562}, col sep=comma] {figures/data/abtesting_coverage_mixture.csv};

    \end{groupplot}
    \end{scope}

\end{tikzpicture}
    \vspace{-0.25cm}
    \caption{Comparison of nominal and robust formulations for linear regression and A/B testing. \emph{Left:} histograms of realized information gains under nominal ($\alpha=1.0$) and robust ($\alpha=0.056$) assumptions. \emph{Right:} expected vs.\ actual coverage for varying values of $\alpha$. Nominal posteriors are clearly overconfident; calibration improves for values $\alpha < 1$. Data obtained from $10^4$ experiments.}
    \label{fig:gain-coverage}
    \vspace{-0.55cm}
\end{figure}

\textbf{Sibson's $\alpha$-Mutual Information.} We compare the \emph{realized} information gain under nominal and robust design criteria. In the nominal setting, the experimenter quantifies information via the Kullback--Leibler divergence in \eqref{eq:eig-param-div}, whereas the robust formulation replaces this measure with a R\'enyi divergence, as in \eqref{eq:risk-sensitive-form}. Figure~\ref{fig:gain-coverage} reports histograms over $10^4$ simulated experiments for both the linear regression and A/B testing settings under uniformly distributed designs. Superficially, the nominal criterion appears to deliver larger gains. This comparison is, however, inherently self-referential. The nominal utility is a \emph{subjective} measure that assumes the nominal likelihood is the true data-generating process, leading the experimenter to overstate the perceived informativeness, whereas the robust utility explicitly accounts for misspecification and yields a more conservative estimate.

\begin{table}[t]
    \centering
    \scriptsize
    \begin{minipage}[c]{0.55\linewidth}
        \centering
        \setlength{\tabcolsep}{3pt}
        \begin{tabular}{l cc cc}
            \hiderowcolors
            \toprule
            & \multicolumn{2}{c}{Linear Regression} & \multicolumn{2}{c}{A/B Testing} \\
            \cmidrule(lr){2-3} \cmidrule(lr){4-5}
            $\alpha$ & Random & Optimal & Random & Optimal \\
            \midrule \rowcolor{gray!30}
            0.056 & $4.62 \times 10^{-1}$ & $\mathbf{2.47 \times 10^{-1}}$ & $1.01 \times 10^{-1}$ & $\mathbf{7.49 \times 10^{-2}}$ \\
            0.115 & $4.78 \times 10^{-1}$ & $\mathbf{2.50 \times 10^{-1}}$ & $9.03 \times 10^{-2}$ & $\mathbf{7.13 \times 10^{-2}}$ \\ \rowcolor{gray!30}
            0.237 & $5.22 \times 10^{-1}$ & $\mathbf{2.53 \times 10^{-1}}$ & $8.98 \times 10^{-2}$ & $\mathbf{7.31 \times 10^{-2}}$ \\
            0.487 & $5.91 \times 10^{-1}$ & $\mathbf{2.56 \times 10^{-1}}$ & $9.77 \times 10^{-2}$ & $\mathbf{8.88 \times 10^{-2}}$ \\ \rowcolor{gray!30}
            1.000 & $6.82 \times 10^{-1}$ & $\mathbf{2.57 \times 10^{-1}}$ & $1.09 \times 10^{-1}$ & $\mathbf{1.06 \times 10^{-1}}$ \\
            \bottomrule
        \end{tabular}
    \end{minipage}%
    \hfill
    \begin{minipage}[c]{0.40\linewidth}
        \vspace{0.15cm}
        \caption{Posterior-mean RMSE for linear regression and A/B testing. We compare the robust posterior under random and optimal designs for varying $\alpha$. Robust designs that optimize Sibson's MI lead to smaller mean squared errors. Data obtained over $10^4$ experiments.}
        \label{tab:mse-combined}
    \end{minipage}

    \vspace{0.30cm}
    \begin{minipage}[c]{0.55\linewidth}
        \centering
        \setlength{\tabcolsep}{3pt}
        \begin{tabular}{l cc ccc ccc}
            \hiderowcolors
            \toprule
            & \multicolumn{2}{c}{Design} & \multicolumn{3}{c}{Sibson RMSE} & \multicolumn{3}{c}{GEIG RMSE} \\
            \cmidrule(lr){2-3} \cmidrule(lr){4-6} \cmidrule(lr){7-9}
            $\alpha$ & Sibson & GEIG & Mean & $P_{10}$ & $P_{90}$ & Mean & $P_{10}$ & $P_{90}$ \\
            \midrule \rowcolor{gray!30}
            0.011 &  0 & 30 & $\mathbf{0.120}$ & 0.034 & 0.208 & 0.132 & 0.036 & 0.231 \\
            0.022 &  0 & 33 & $\mathbf{0.098}$ & 0.030 & 0.167 & 0.112 & 0.033 & 0.193 \\ \rowcolor{gray!30}
            0.041 &  0 & 35 & $\mathbf{0.081}$ & 0.025 & 0.143 & 0.093 & 0.028 & 0.158 \\
            0.077 &  0 & 37 & $\mathbf{0.072}$ & 0.022 & 0.141 & 0.079 & 0.025 & 0.146 \\ \rowcolor{gray!30}
            0.147 &  0 & 39 & $\mathbf{0.071}$ & 0.020 & 0.144 & 0.076 & 0.024 & 0.145 \\
            0.278 & 18 & 42 & $\mathbf{0.075}$ & 0.022 & 0.150 & 0.083 & 0.028 & 0.154 \\ \rowcolor{gray!30}
            0.527 & 33 & 43 & $\mathbf{0.091}$ & 0.030 & 0.163 & 0.095 & 0.031 & 0.169 \\
            1.000 & 41 & 42 & 0.106 & 0.036 & 0.184 & 0.106 & 0.037 & 0.184 \\
            \bottomrule
        \end{tabular}
    \end{minipage}%
    \hfill
    \begin{minipage}[c]{0.40\linewidth}
        \vspace{0.15cm}
        \caption{Comparison of Sibson's and GEIG's optimal designs and RMSE on an A/B testing problem with $N_x = 100$. The two objectives lead to different design choices. We report the optimal allocation, mean RMSE, the $10$th and $90$th percentiles for each method. Evaluation performed over $10^4$ experiments.}
        \label{tab:sibson-vs-geig}
    \end{minipage}
    \vspace{-0.40cm}
\end{table}
This overconfidence manifests at two distinct levels. \emph{At the inference stage}, the coverage analysis in the right panel of Figure~\ref{fig:gain-coverage} compares expected credible levels with the empirical frequency at which the true parameter falls within the corresponding credible sets. For the nominal Bayesian posterior ($\alpha = 1$), coverage curves lie well below the diagonal, indicating severe undercoverage. The robust tilted posterior progressively improves calibration as $\alpha$ decreases, with curves moving closer to or above the diagonal reflecting a more conservative stance compared to the nominal Bayes posterior.

{\color{black}
\emph{At design selection}, Table~\ref{tab:mse-combined} reports the root mean squared error (RMSE) of the posterior mean under random and optimal designs across a geometrically spaced grid of $\alpha$ values. Optimal designs produce posterior means that are closer to the true parameter. In both the linear regression and A/B testing settings, there exists an $\alpha < 1$ for which the robust optimal design outperforms the nominal optimum at $\alpha = 1$. Overall, these results confirm that Sibson's $\alpha$-mutual information favors designs that are less sensitive to model misspecification and leads to better calibrated posteriors.

Additionally, in Appendix~\ref{app:compare-barlas}, we give a detailed comparison between Sibson's $\alpha$-mutual information and the \emph{Gibbs expected information gain} (GEIG) proposed by \citet{barlas2025robust}. Although both objectives are connected to tilted Gibbs posteriors, we show that they are not equivalent in general. Table~\ref{tab:sibson-vs-geig} includes an empirical comparison across $\alpha$ values, demonstrating that the two objectives can lead to different optimal designs on the A/B testing problem.

\textbf{PAC-Bayes Policies.} While the preceding evaluation relied on closed-form formulas for computation, this section focuses on optimizing designs using the nested Monte Carlo estimator, thereby foregoing any benefits of analytical tractability. To test design optimization in a higher-dimensional setting, we extend the linear regression problem to a 10-dimensional parameter and design space, while A/B testing accommodates $N_x = 100$ subjects. In both cases, the prior policy $\pi_0$ is taken to be uniform over the design space and the precision is tuned over a grid and set to $\lambda = 10^6$; see Table~\ref{tab:linreg-lambda-sweep}.
}
\begin{table}[t]
    \scriptsize
    \begin{minipage}[t]{0.48\linewidth}
        \centering
        \caption{Regret comparison on a linear regression between a naive optimizer and a PAC-Bayes policy for $M=16$ and variable $N$. We report the mean regret together with the $10$th and $90$th percentiles, computed over 256 repetitions.}
        \vspace{-0.2cm}
        \label{tab:linreg-sample-sweep}
        \setlength{\tabcolsep}{3pt}
        \begin{tabular}{l ccc ccc}
            \hiderowcolors
            \toprule
            $N$ & \multicolumn{3}{c}{Naive} & \multicolumn{3}{c}{PAC-Bayes} \\
            \cmidrule(lr){2-4} \cmidrule(lr){5-7}
            & Mean & $P_{10}$ & $P_{90}$ & Mean & $P_{10}$ & $P_{90}$ \\
            \midrule \rowcolor{gray!30}
            16  & 0.327 & 0.219 & 0.430 & $\mathbf{0.077}$ & $\mathbf{0.058}$ & $\mathbf{0.099}$ \\
            32  & 0.312 & 0.212 & 0.413 & $\mathbf{0.083}$ & $\mathbf{0.040}$ & $\mathbf{0.135}$ \\ \rowcolor{gray!30}
            64  & 0.307 & 0.205 & 0.407 & $\mathbf{0.042}$ & $\mathbf{0.031}$ & $\mathbf{0.054}$ \\
            128 & 0.296 & 0.200 & 0.386 & $\mathbf{0.041}$ & $\mathbf{0.036}$ & $\mathbf{0.046}$ \\ \rowcolor{gray!30}
            256 & 0.282 & 0.190 & 0.370 & $\mathbf{0.018}$ & $\mathbf{0.013}$ & $\mathbf{0.023}$ \\
            \bottomrule
        \end{tabular}
    \end{minipage}\hfill
    \begin{minipage}[t]{0.48\linewidth}
        \centering
        \caption{Relative regret comparison on a linear regression between a naive optimizer and a PAC-Bayes policy for varying $\alpha$ values. We report the mean relative regret together with the $10$th and $90$th percentiles, computed over 256 repetitions.}
        \vspace{-0.2cm}
        \label{tab:linreg-alpha-sweep}
        \setlength{\tabcolsep}{4pt}
        \begin{tabular}{l ccc ccc}
            \hiderowcolors
            \toprule
            $\alpha$ & \multicolumn{3}{c}{Naive} & \multicolumn{3}{c}{PAC-Bayes} \\
            \cmidrule(lr){2-4} \cmidrule(lr){5-7}
            & Mean & $P_{10}$ & $P_{90}$ & Mean & $P_{10}$ & $P_{90}$ \\
            \midrule \rowcolor{gray!30}
            0.056 & 0.098 & 0.066 & 0.128 & $\mathbf{0.023}$ & $\mathbf{0.014}$ & $\mathbf{0.031}$ \\
            0.115 & 0.100 & 0.066 & 0.132 & $\mathbf{0.020}$ & $\mathbf{0.013}$ & $\mathbf{0.028}$ \\ \rowcolor{gray!30}
            0.237 & 0.103 & 0.072 & 0.138 & $\mathbf{0.016}$ & $\mathbf{0.012}$ & $\mathbf{0.021}$ \\
            0.487 & 0.104 & 0.070 & 0.143 & $\mathbf{0.025}$ & $\mathbf{0.016}$ & $\mathbf{0.035}$ \\ \rowcolor{gray!30}
            1.000 & 0.105 & 0.066 & 0.147 & $\mathbf{0.025}$ & $\mathbf{0.016}$ & $\mathbf{0.036}$ \\
            \bottomrule
        \end{tabular}
    \end{minipage}
    \vspace{-0.15cm}
\end{table}
\begin{figure}[t]
    \centering
    \begin{minipage}[c]{0.60\textwidth}
        \centering
        \begin{tikzpicture}
    \begin{groupplot}[
        group style={
            group name=policy_plots,
            group size=2 by 2,
            horizontal sep=0.5cm,
            vertical sep=0.65cm,
        },
        width=5cm,
        height=3cm,
        xlabel={},
        ylabel={Frequency},
        ymin=0,
        enlarge x limits=false,
        scaled y ticks=base 10:-2,
        ytick scale label code/.code={$\times 10^{2}$},
        y tick scale label style={
            at={(axis description cs:0.05,1.025)},
            anchor=south west,
            xshift=-7pt,
            yshift=0pt,
            font=\scriptsize,
        },
        label style={font=\scriptsize},
        tick label style={font=\scriptsize},
        title style={font=\footnotesize},
        grid=major,
        grid style={gray!30},
        xlabel style={yshift=2pt},
        ylabel style={yshift=-5pt},
        title style={yshift=-5pt},
        legend style={
            draw,
            font=\tiny,
            /pgfplots/legend cell align=left,
            inner sep=1.pt,
            outer sep=0pt,
        },
    ]
    
    \nextgroupplot[
        ymax=1100,
        xmin=0.0,
        xmax=0.6,
        xlabel={},
        legend style={
            at={(0.97,0.97)},
            anchor=north east,
            draw,
            font=\tiny,
            /pgfplots/legend cell align=left,
            inner sep=1.0pt,
            outer sep=0pt,
            legend image post style={scale=1.0},
            text depth=0pt,
            row sep=-1pt,
        },
    ]
    \addplot[
        ybar interval,
        fill=gray!40,
        draw=black,
        line width=0.25pt,
        fill opacity=0.7,
        area legend,
    ] table[
        x=bin_center,
        y=gd_density,
        col sep=comma
    ] {figures/data/linreg_regret_merged.csv};
    \addlegendentry{Naive}
    
    \addplot[
        ybar interval,
        fill=gray!150,
        draw=black,
        line width=0.25pt,
        fill opacity=0.7,
        area legend,
    ] table[
        x=bin_center,
        y=reps_density,
        col sep=comma
    ] {figures/data/linreg_regret_merged.csv};
    \addlegendentry{PAC-Bayes}
    
    \nextgroupplot[
        ymax=1100,
        xmin=0.4,
        xmax=1.04,
        xtick={0.4,0.6,0.8,1.0},
        xlabel={},
        ylabel={},
        yticklabels={},  %
        ytick scale label code/.code={},  %
    ]
    \addplot[
        ybar interval,
        fill=gray!40,
        draw=black,
        line width=0.25pt,
        fill opacity=0.7,
        area legend,
    ] table[
        x=bin_center,
        y=gd_density,
        col sep=comma
    ] {figures/data/linreg_similarity_merged.csv};

    \addplot[
        ybar interval,
        fill=gray!150,
        draw=black,
        line width=0.25pt,
        fill opacity=0.7,
        area legend,
    ] table[
        x=bin_center,
        y=reps_density,
        col sep=comma
    ] {figures/data/linreg_similarity_merged.csv};
    
    \nextgroupplot[
        ymax=1150,
        xmin=0.0,
        xmax=0.3,
        xtick={0, 0.1, 0.2, 0.3},
        xlabel={Objective Regret},
    ]
    \addplot[
        ybar interval,
        fill=gray!40,
        draw=black,
        line width=0.25pt,
        fill opacity=0.7,
        area legend,
    ] table[
        x=bin_center,
        y=enum_density,
        col sep=comma
    ] {figures/data/abtesting_regret_merged.csv};
    
    \addplot[
        ybar interval,
        fill=gray!150,
        draw=black,
        line width=0.25pt,
        fill opacity=0.7,
        area legend,
    ] table[
        x=bin_center,
        y=reps_density,
        col sep=comma
    ] {figures/data/abtesting_regret_merged.csv};
    
    \nextgroupplot[
        ymax=1150,
        xmin=0.7,
        xmax=1.02,
        xtick={0.7,0.8,0.9,1.0},
        xlabel={Design Optimality},
        ylabel={},
        yticklabels={},  %
        ytick scale label code/.code={},  %
    ]
    \addplot[
        ybar interval,
        fill=gray!40,
        draw=black,
        line width=0.25pt,
        fill opacity=0.7,
        area legend,
    ] table[
        x=bin_center,
        y=enum_density,
        col sep=comma
    ] {figures/data/abtesting_similarity_merged.csv};

    \addplot[
        ybar interval,
        fill=gray!150,
        draw=black,
        line width=0.25pt,
        fill opacity=0.7,
        area legend,
    ] table[
        x=bin_center,
        y=reps_density,
        col sep=comma
    ] {figures/data/abtesting_similarity_merged.csv};
    
    \end{groupplot}
\end{tikzpicture}
    \end{minipage}\hfill
    \begin{minipage}[c]{0.38\textwidth}
        \vspace{-0.30cm}
        \caption{Empirical distributions of regret (left) and design optimality (right) for linear regression (top) and A/B testing (bottom). We compare naive optimization against our PAC-Bayes policy. Naive optimization has higher regret and variability. Its designs are suboptimal, reflected in lower similarity scores relative to the theoretical optimal design. Data from $1024$ trials.}
        \label{fig:pac-hist}
    \end{minipage}
    \vspace{-0.6cm}
\end{figure}

Figure~\ref{fig:pac-hist} depicts the failure of naive design optimization when the objective is estimated by a noisy, biased oracle. We compare stochastic gradient descent for linear regression and exhaustive enumeration for A/B testing against a PAC-Bayes stochastic policy optimized via mirror descent, while fixing $\alpha$, $\lambda$, and the sample sizes $N$ and $M$. We initialize $1024$ distinct naive optimizers and evaluate the realized regret and design optimality relative to the tractable optimum. The histograms show that naive optimizers incur higher and more variable regret, reflecting suboptimal designs, whereas the PAC-Bayes policy concentrates near the optimum and achieves consistently lower regret.

Finally, we examine the impact of sample complexity in Table~\ref{tab:linreg-sample-sweep}, which reports the mean regret alongside the $10$th and $90$th percentiles across varying outer sample sizes $N$. While the performance of the naive optimizer marginally improves as $N$ increases, it fails to match the performance of the PAC-Bayes policy. Table~\ref{tab:linreg-alpha-sweep} confirms that this performance gap holds uniformly across different values of $\alpha$. For fixed sample sizes $N$ and $M$, the naive optimizer consistently incurs higher regret than the PAC-Bayes policy, validating the PAC-Bayesian perspective.

\section{Discussion}
\label{sec:discussion}

We established a principled framework for robust Bayesian experimental design using the maximin principle. We proposed Sibson's $\alpha$-mutual information as a robust alternative to the standard expected information gain and showed that the resulting belief update is an $\alpha$-tilted posterior, while the conditional information gain corresponds to a R\'enyi divergence between posterior and prior. To handle the intractability of Sibson's $\alpha$-MI and the finite-sample Monte Carlo error, we introduced PAC-Bayesian stochastic policies and proved a lower bound for the true robust objective. Together, these results characterize maximin robust Bayesian experimental design under information-theoretic constraints, including the adversarial perturbation, robust information gain, and tilted posterior update.

As for limitations, although the maximin formulation provides worst-case guarantees, it may lead to overly conservative designs~\citep{watson2016approximate}. Exploring alternative robust decision-making paradigms could complement our work. In addition, calibrating the misspecification order $\alpha$ as experimental data accumulates remains an open challenge. Related calibration approaches in the generalized Bayesian inference literature may offer useful intuition~\citep{wu2021calibrating}. A further research direction is robust \emph{sequential} experimental design, where the experimenter plans over a horizon of experiments rather than myopically optimizing a single design~\citep{foster2021deep}.

\newpage
\begin{ack}
H. Abdulsamad and C. A. Naesseth acknowledge the support of the UvA--Bosch Delta Lab and the Amsterdam Machine Learning Lab (AMLab) at the University of Amsterdam. This project was supported by the Bosch Center for Artificial Intelligence. The authors declare no competing interests.
\end{ack}

\bibliographystyle{plainnat}
\bibliography{bibliography}

\newpage

\appendix

\section{Organization of the Appendix}

The Appendix is organized as follows:
\begin{itemize}[leftmargin=2em]
    \item Appendix~\ref{app:proofs} provides proofs for the various lemmata and propositions.
    \item Appendix~\ref{app:nested-estimator} provides details on the nested Monte Carlo estimator.
    \item Appendix~\ref{app:pac-bayes-advantage} compares PAC-Bayesian and deterministic policies.
    \item Appendix~\ref{app:pac-bayes-precision} discusses the role and choice of the PAC-Bayes precision.
    \item Appendix~\ref{app:evaluation-details} provides details for the numerical evaluations.
    \item Appendix~\ref{app:compare-barlas} explores the connection to \citet{barlas2025robust}.
\end{itemize}

\section{Proofs}
\label{app:proofs}

\subsection{Proof of Lemma~\ref{lem:minimization-decomposition}}
\label{app:proof-minimization-decomposition}

We consider the worst-case expected utility minimization problem defined by:
\begin{align}
    \inf_{q} \, V(\pi, q) & = \inf_{q} \, \mathbb{E}_{\pi} \left[ U(\xi, q) \right] \\
    & = \inf_{q} \, \mathbb{E}_{\pi} \left[ \kld{q(\theta, x \mid \xi)}{p(\theta) \, q(x \mid \xi)} \right],
\end{align}
subject to the Kullback--Leibler constraint averaged under $\pi(\xi)$:
\begin{equation}
    \mathbb{E}_{\pi} \left[ \kld{q(\theta, x \mid \xi)}{p(\theta, x \mid \xi)} \right] \le \rho.
\end{equation}
Given that both the objective and the constraint are convex in $q(\cdot \mid \xi)$, and under strict feasibility assumptions, strong duality holds~\citep{boyd2004convex}. Therefore, we can write:
\begin{equation}
    \inf_{q \in \mathcal{Q}_{\rho}} V(\pi, q) = \sup_{\beta > 0} \, \inf_{q} \, \mathcal{L}(\beta, \pi, q),
\end{equation}
where $\mathcal{L}$ is the Lagrangian functional given by:
\begin{equation}
    \mathcal{L}(\beta, \pi, q) \coloneq \mathbb{E}_{\pi} \left[ \kld{q(\theta, x \mid \xi)}{p(\theta) \, q(x \mid \xi)} \right] + \beta \, \mathbb{E}_{\pi} \left[ \kld{q(\theta, x \mid \xi)}{p(\theta, x \mid \xi)} \right] - \beta \rho.
\end{equation}
As a result, the worst-case expected utility is an expectation over regularized pointwise objectives:
\begin{equation}
    \inf_{q \in \mathcal{Q}_{\rho}} V(\pi, q) = \sup_{\beta > 0} \, \left\{ \mathbb{E}_{\pi} \left[ \mathcal{J}_{\beta}(\xi) \right] - \beta \rho \right\},
\end{equation}
where $\mathcal{J}_{\beta}(\xi)$ is the regularized worst-case utility:
\begin{equation}
    \mathcal{J}_{\beta}(\xi) \coloneq \inf_{q} \, \left\{ U(\xi, q) + \beta \, \kld{q(\theta, x \mid \xi)}{p(\theta, x \mid \xi)} \right\}.
\end{equation}
\qed

\subsection{Proof of Proposition~\ref{prop:robust-expected-info-gain}}
\label{app:proof-sibson-mutual-info}

We start by stating a useful lemma that highlights a variational interpretation of the mutual information.

\begin{lemma}[Shannon's variational mutual information, \citealp{verdu2015alpha}]
    \label{lem:variational-mutual-info}
    Shannon's mutual information between $\theta$ and $x$ given $\xi$ is the minimal value of the Kullback--Leibler divergence between the true joint distribution and the product of variational marginals:
    \begin{equation}
        \label{eq:kullback-leibler-mutual-info}
        I(\xi) \coloneq \inf_{\nu \in \mathcal{P}(\mathcal{X})} \, \kld{p(\theta, x \mid \xi)}{p(\theta) \, \nu(x \mid \xi)}.
    \end{equation}
    The infimum is attained when the variational marginal coincides with the true marginal, so that $\nu^{\star}(x \mid \xi) = p(x \mid \xi)$. At this optimum, we obtain:
    \begin{equation}
        I(\xi) = \kld{p(\theta, x \mid \xi)}{p(\theta) \, p(x \mid \xi)}.
    \end{equation}
\end{lemma}

Now, we start from the definition of $\mathcal{J}_{\beta}(\xi)$ for any $\beta > 0$:
\begin{equation}
    \mathcal{J}_{\beta}(\xi) \coloneq \inf_{q} \, \left\{ U(\xi, q) + \beta \, \kld{q(\theta, x \mid \xi)}{p(\theta, x \mid \xi)} \right\}.
\end{equation}
The expression for $U(\xi, q)$ contains the predictive marginal $q(x \mid \xi)$, which changes with the optimized joint law $q(\theta, x \mid \xi)$. To externalize this marginal dependence, let us first make use of the chain rule for relative entropy,
\begin{align}
    U(\xi, q) & = \kld{q(\theta, x \mid \xi)}{p(\theta) \, q(x \mid \xi)} \\
    & = \kld{q(\theta, x \mid \xi)}{q(\theta \mid \xi) \, q(x \mid \xi)} + \kld{q(\theta \mid \xi)}{p(\theta)},
\end{align}
then apply the variational form from Lemma~\ref{lem:variational-mutual-info} to the first term:
\begin{equation}
    \kld{q(\theta, x \mid \xi)}{q(\theta \mid \xi) \, q(x \mid \xi)} = \inf_{\nu} \, \kld{q(\theta, x \mid \xi)}{q(\theta \mid \xi) \, \nu(x \mid \xi)},
\end{equation}
with $\nu^\star(x \mid \xi) = q(x \mid \xi)$. Therefore, we can write the payoff as follows:
\begin{align}
    U(\xi, q) & = \inf_{\nu} \, \kld{q(\theta, x \mid \xi)}{q(\theta \mid \xi) \, \nu(x \mid \xi)} + \kld{q(\theta \mid \xi)}{p(\theta)} \\
    & = \inf_{\nu} \, \kld{q(\theta, x \mid \xi)}{p(\theta) \, \nu(x \mid \xi)}.
\end{align}
This trick temporarily decouples the marginal from the optimized joint distribution. Now, we can write the following minimization problem:
\begin{equation}
    \mathcal{J}_{\beta}(\xi) = \inf_{q} \, \inf_{\nu} \left\{ \kld{q(\theta, x \mid \xi)}{p(\theta) \, \nu(x \mid \xi)} + \beta \, \kld{q(\theta, x \mid \xi)}{p(\theta, x \mid \xi)} \right\},
\end{equation}
with implicit constraints to ensure that $q(\theta, x \mid \xi)$ and $\nu(x \mid \xi)$ are proper densities.

Since $q$ and $\nu$ are optimized over independent sets, $\mathcal{P}(\Theta \times \mathcal{X})$ and $\mathcal{P}(\mathcal{X})$, the objective may be viewed as a functional on their product space, and the nested minimization commutes to $\inf_{\nu} \inf_{q}$. Furthermore, for a fixed marginal $\nu$, the minimization over $q$ is convex with a convex constraint and strong duality holds under strict feasibility assumptions~\citep{boyd2004convex}.

To form the associated Lagrangian functional, we introduce the multipliers $\eta$ and $\kappa$ associated with the normalization constraints. The resulting dual problem is:
\begin{equation}
    \mathcal{J}_{\beta}(\xi) = \sup_{\kappa} \, \inf_{\nu} \, \sup_{\eta} \, \inf_{q} \, \mathcal{G}(\beta, \eta, q, \kappa, \nu, \xi),
\end{equation}
where the Lagrangian functional $\mathcal{G}$ is given by:
\begin{align}
    \label{eq:lagrangian}
    \mathcal{G}(\beta, \eta, q, \kappa, \nu, \xi) & =
    \begin{aligned}[t]
        & \int q(\theta, x \mid \xi) \log \frac{q(\theta,x \mid \xi)}{p(\theta)\,\nu(x \mid \xi)} \dif \theta \dif x \\
        & + \beta \, \int q(\theta, x \mid \xi) \log \frac{q(\theta,x \mid \xi)}{p(\theta, x \mid \xi)} \dif \theta \dif x \\
        & + \eta \left[ \int q(\theta, x \mid \xi) \dif \theta \dif x - 1 \right] + \kappa \left[ \int \nu(x \mid \xi) \dif x - 1 \right].
    \end{aligned}
\end{align}
We start by minimizing $\mathcal{G}$ with respect to the joint distribution $q(\cdot \mid \xi)$. We take the first variation of $\mathcal{G}$ and set it to zero to obtain a stationarity condition:
\begin{equation}
(1 + \beta) \left[ \log q^{\star}(\theta, x \mid \xi) + 1 \right] - \log p(\theta) - \log \nu(x \mid \xi) - \beta \log p(\theta, x \mid \xi) + \eta = 0.
\end{equation}
Solving for $q^{\star}(\cdot \mid \xi)$:
\begin{equation}
q^{\star}(\theta, x \mid \xi) = \exp \left\{ - \frac{\eta}{1 + \beta} - 1 \right\} \Big[ p(\theta) \, \nu(x \mid \xi) \Big]^{1 / (1 + \beta)} \, \Big[ p(\theta, x \mid \xi) \Big]^{\beta / (1 + \beta)}
\end{equation}
and accounting for the optimal normalization multiplier $\eta^{\star}$, we get the geometric mixture:
\begin{equation}
q^{\star}(\theta, x \mid \xi) =
    \frac{\Big[ p(\theta) \, \nu(x \mid \xi) \Big]^{1 / (1 + \beta)} \, \Big[ p(\theta, x \mid \xi) \Big]^{\beta / (1 + \beta)}}{\displaystyle \int \Big[ p(\theta) \, \nu(x \mid \xi) \Big]^{1 / (1 + \beta)} \, \Big[ p(\theta, x \mid \xi) \Big]^{\beta / (1 + \beta)} \, \dif \theta \dif x},
\end{equation}
where $\beta / (1 + \beta)$ interpolates geometrically between the factorized reference and the nominal joint. Substituting the optimum $q^{\star}(\cdot \mid \xi)$ back into the Lagrangian~\eqref{eq:lagrangian} eliminates the dependence and reduces the functional to depend only on the variational marginal $\nu(\cdot \mid \xi)$ and the multipliers $\beta$ and $\kappa$:
\begin{align}
    &\,\mathcal{G}(\beta, \kappa, \nu, \xi) \\
    & = - (1 + \beta) \log \left\{ \int \Big[ p(\theta) \, \nu(x \mid \xi) \Big]^{1 / (1 + \beta)} \, \Big[ p(\theta, x \mid \xi) \Big]^{\beta / (1 + \beta)} \, \dif \theta \dif x \right\} + \kappa \left[ \int \nu(x \mid \xi) \dif x - 1 \right] \\
    & = - (1 + \beta) \log \left\{ \int \Big[ p(\theta) \nu(x \mid \xi) \Big]^{1 / (1 + \beta)} \Big[ p(\theta) p(x \mid \theta, \xi) \Big]^{\beta / (1 + \beta)} \, \dif \theta \dif x \right\} + \kappa \left[ \int \nu(x \mid \xi) \dif x - 1 \right] \\
    & = - (1 + \beta) \log \left\{ \int \Big[ \nu(x \mid \xi) \Big]^{1 / (1 + \beta)} f_{\beta}(x \mid \xi) \dif x \right\} + \kappa \left[ \int \nu(x \mid \xi) \dif x - 1 \right],
\end{align}
where we have introduced a shorthand for the \emph{soft} marginal likelihood:
\begin{equation}
    f_{\beta}(x \mid \xi) = \int p(\theta) \, \Big[ p(x \mid \theta, \xi) \Big]^{\beta / (1 + \beta)} \, \dif \theta.
\end{equation}
We now consider the first variation of $\mathcal{G}$ with respect to $\nu(\cdot \mid \xi)$, yielding the stationarity condition:
\begin{equation}
    \kappa - \frac{f_{\beta}(x \mid \xi) \, \Big[ \nu^{\star}(x \mid \xi) \Big]^{- \beta / (1 + \beta)}}{\displaystyle \int \Big[ \nu^{\star}(x \mid \xi) \Big]^{1 / (1 + \beta)} f_{\beta}(x \mid \xi) \dif x} = 0.
\end{equation}
Similar to $q(\cdot \mid \xi)$, solving for $\nu^{\star}(\cdot \mid \xi)$ while accounting for the optimal normalization multiplier $\kappa^{\star}$ leads to:
\begin{align}
    \nu^{\star}(x \mid \xi) & = \frac{ \Big[ f_{\beta}(x \mid \xi) \Big]^{(1 + \beta)/\beta}}{\displaystyle \int \Big[ f_{\beta}(x \mid \xi) \Big]^{(1 + \beta)/\beta} \dif x} = \frac{ \Bigg[ \displaystyle \int p(\theta) \, \Big[ p(x \mid \theta, \xi) \Big]^{\beta / (1 + \beta)} \, \dif \theta \Bigg]^{(1 + \beta)/\beta}}{\displaystyle \int \Bigg[ \displaystyle \int p(\theta) \, \Big[ p(x \mid \theta, \xi) \Big]^{\beta / (1 + \beta)} \, \dif \theta \Bigg]^{(1 + \beta)/\beta} \dif x} \coloneq p_{_{\!\beta}\!}(x \mid \xi).
\end{align}
This distribution is a tilted marginal induced by the constraint. Its dependence on $\beta$ encodes how aggressively mass is concentrated on high-likelihood regions. Substitute this result back into $\mathcal{G}$ to get the dual as a function of $\beta$ only:
\begin{align}
    \mathcal{G}(\beta, \xi) & = - \beta \, \log \left\{ \int \Big[ f_{\beta}(x \mid \xi) \Big]^{(1 + \beta) / \beta} \dif x \right\} \\
    & = - \beta \, \log \left\{ \int \Bigg[ \int p(\theta) \, \Big[ p(x \mid \theta, \xi) \Big]^{\beta / (1 + \beta)} \, \dif \theta \Bigg]^{(\beta + 1) / \beta} \dif x \right\} \\
    & = - \beta \, \log \left\{ \int \Big [p(\theta, x \mid \xi) \Big]^{\beta / (1 + \beta)} \, \Big[ p(\theta) \, p_{_{\!\beta}\!}(x \mid \xi) \Big]^{1 / (1 + \beta)} \, \dif \theta \dif x \right\} \\
    & = \mathbb{D}_{\textstyle \!\frac{\beta}{1 + \beta}} \big[ p(\theta, x \mid \xi) \,\|\, p(\theta) \, p_{_{\!\beta}\!}(x \mid \xi) \big] = \inf_{\nu} \, \mathbb{D}_{\textstyle \!\frac{\beta}{1 + \beta}} \big[ p(\theta, x \mid \xi) \,\|\, p(\theta) \, \nu(x \mid \xi) \big],
\end{align}
where $\mathbb{D}_{\zeta}$ is the R\'enyi divergence of order $\zeta$. We substitute $\alpha = \beta / (1 + \beta)$ with $\alpha \in (0, 1)$:
\begin{equation}
    \mathcal{J}_{\beta}(\xi) = \alphad{p(\theta, x \mid \xi)}{p(\theta) \, p_{_{\!\alpha}\!}(x \mid \xi)} \coloneq I^{\timesit{S}}_{\alpha}(\theta; x)(\xi).
\end{equation}
Finally, the parameter $\beta$ can be optimized according to the global objective from Lemma~\ref{lem:minimization-decomposition}:
\begin{equation}
    \beta^{\star} \coloneq \arg\sup_{\beta > 0} \left\{ \mathbb{D}_{\textstyle \!\frac{\beta}{1 + \beta}} \big[ p(\theta, x \mid \xi) \,\|\, p(\theta) \, p_{_{\!\beta}\!}(x \mid \xi) \big] - \beta \, \rho \right\}.
\end{equation}
\qed

\subsection{Proof of Proposition~\ref{prop:robust-conditional-info-gain}}
\label{app:proof-robust-conditional-info-gain}

We start from the definition of the R\'enyi divergence of order $\alpha$ for distributions $P$ and $Q$:
\begin{equation}
    \alphad{P}{Q} = \frac{1}{\alpha - 1} \log \int P(z)^{\alpha} \, Q(z)^{1 - \alpha} \, \dif z.
\end{equation}
For the robust expected information gain~\eqref{eq:robust-expected-info-gain}, we have $P = p(\theta) \, p(x \mid \theta, \xi)$ and $Q = p(\theta) \, p_{\alpha}(x \mid \xi)$. Substituting leads to:
\begin{align}
    I_{\alpha}^{\timesit{S}}(\theta; x)(\xi) & = \frac{1}{\alpha - 1} \log \iint \Big[ p(\theta) \, p(x \mid \theta, \xi) \Big]^{\alpha} \Big[ p(\theta) \, p_{\alpha}(x \mid \xi) \Big]^{1 - \alpha} \dif \theta \, \dif x \\
    & = \frac{1}{\alpha - 1} \log \int p_{\alpha}(x \mid \xi)^{1 - \alpha} \int p(\theta) \, p(x \mid \theta, \xi)^{\alpha} \dif \theta \, \dif x.
\end{align}
The optimal tilted marginal $p_{\alpha}(x \mid \xi)$ is defined by the $\alpha$-tilted marginal:
\begin{equation}
    p_{\alpha}(x \mid \xi) = \frac{1}{Z} \left[ \int p(\theta) \, p(x \mid \theta, \xi)^{\alpha} \dif \theta \right]^{1/\alpha}, \quad Z = \int \left[ \int p(\theta) \, p(x \mid \theta, \xi)^{\alpha} \dif \theta \right]^{1/\alpha} \dif x.
\end{equation}
Substituting this form back into the integral, we obtain:
\begin{align}
    I_{\alpha}^{\timesit{S}}(\theta; x)(\xi) & = \frac{1}{\alpha - 1} \, \log \left\{ Z^{\alpha - 1} \,\int \left[ \int p(\theta) \, p(x \mid \theta, \xi)^{\alpha} \dif \theta \right]^{1/\alpha} \dif x \right\} \\
    & = \frac{\alpha}{\alpha - 1} \log \int \left[ \int p(\theta) \, p(x \mid \theta, \xi)^{\alpha} \dif \theta \right]^{1/\alpha} \dif x.
\end{align}
Next, we multiply and divide the inner term by $p(x \mid \xi)^{\alpha}$:
\begin{align}
    I_{\alpha}^{\timesit{S}}(\theta; x)(\xi) & = \frac{\alpha}{\alpha - 1} \log \int p(x \mid \xi)  \left[ \int p(\theta) \, \frac{p(x \mid \xi, \theta)^{\alpha}}{p(x \mid \xi)^{\alpha}} \dif \theta \right]^{1 / \alpha} \dif x \\
    & = \frac{\alpha}{\alpha - 1} \log \int p(x \mid \xi)  \left[ \int p(\theta) \, \left[ \frac{p(\theta \mid x, \xi)}{p(\theta)} \right]^{\alpha} \dif \theta \right]^{1 / \alpha} \dif x \\
    & = \frac{\alpha}{\alpha - 1} \log \int p(x \mid \xi)  \left[\int p(\theta)^{1 - \alpha} \, p(\theta \mid x, \xi)^{\alpha} \dif \theta \right]^{1 / \alpha} \dif x.
\end{align}
The inner integral is the scaled exponential of the R\'enyi divergence between posterior and prior:
\begin{align}
    I_{\alpha}^{\timesit{S}}(\theta; x)(\xi) & = \frac{\alpha}{\alpha - 1} \log \int p(x \mid \xi) \exp \left\{ \frac{\alpha - 1}{\alpha} \alphad{p(\theta \mid x, \xi)}{p(\theta)} \right\} \dif x.
\end{align}

To highlight risk-sensitivity, we expand this representation around $\alpha = 1$. Let $\delta \coloneq \alpha - 1$ and let:
\begin{equation}
    G_{\alpha}(x, \xi) \coloneq \alphad{p(\theta \mid x, \xi)}{p(\theta)}, \quad L(\theta, x, \xi) \coloneq \log\,\left[p(\theta \mid x, \xi) / p(\theta)\right].
\end{equation}
We first expand the conditional R\'enyi divergence $G_{\alpha}(x, \xi)$:
\begin{align}
    G_{\alpha}(x, \xi) & = \frac{1}{\alpha - 1} \log \int p(\theta \mid x, \xi) \, \left[\frac{p(\theta \mid x, \xi)}{p(\theta)} \right]^{\alpha - 1} \dif \theta \\
    & = \frac{1}{\delta} \log \mathbb{E}_{p(\theta \mid x, \xi)} \! \left[ \exp \Big\{ \delta \, L(\theta, x, \xi) \Big\} \right] \\
    & = \mathbb{E}_{p(\theta \mid x, \xi)} \left[ L(\theta, x, \xi) \right] + \frac{\delta}{2} \, \mathbb{V}_{p(\theta \mid x, \xi)} \! \left[ L(\theta, x, \xi) \right] + \mathcal{O}(\delta^{2}).
\end{align}
Next, we set $\eta \coloneq (\alpha - 1)/\alpha$. Since $\delta \coloneq \alpha - 1$, we have:
\begin{equation}
    \eta = \frac{\delta}{1 + \delta} = \delta + \mathcal{O}(\delta^2).
\end{equation}
Expanding the outer exponential average gives:
\begin{align}
    I_{\alpha}^{\timesit{S}}(\theta; x)(\xi) & = \frac{1}{\eta} \log \mathbb{E}_{p(x \mid \xi)} \! \left[ \exp \Big\{ \eta \, G_{\alpha}(x, \xi) \Big\} \right] \\
    & = \mathbb{E}_{p(x \mid \xi)}\left[G_{\alpha}(x, \xi)\right] + \frac{\eta}{2} \, \mathbb{V}_{p(x \mid \xi)} \! \left[G_{\alpha}(x, \xi)\right] + \mathcal{O}(\eta^{2}).
\end{align}
Next, we substitute the expansion of $G_{\alpha}(x, \xi)$ into the expectation and variance terms with respect to $p(x \mid \xi)$, so that:
\begin{align}
    \mathbb{E}_{p(x \mid \xi)}[G_{\alpha}(x, \xi)] & = I(\xi) + \frac{\delta}{2} \, \mathbb{E}_{p(x \mid \xi)} \left[ \mathbb{V}_{p(\theta \mid x, \xi)} \! \left[L(\theta, x, \xi)\right] \right] + \mathcal{O}(\delta^{2}), \\
    \mathbb{V}_{p(x \mid \xi)}[G_{\alpha}(x, \xi)] & = \mathbb{V}_{p(x \mid \xi)} \! \left[ \kld{p(\theta \mid x, \xi)}{p(\theta)} \right] + \mathcal{O}(\delta),
\end{align}
where we used $I(\xi) = \mathbb{E}_{p(x \mid \xi)} \left[ \kld{p(\theta \mid x, \xi)}{p(\theta)} \right]$. Substituting both back into $I_{\alpha}^{\timesit{S}}(\theta; x)(\xi)$:
\begin{equation}
    I_{\alpha}^{\timesit{S}}(\theta; x)(\xi) = I(\xi) + \frac{\delta}{2}\mathbb{E}_{p(x \mid \xi)} \! \left[ \mathbb{V}_{p(\theta \mid x, \xi)} \left[L(\theta, x, \xi)\right] \right] + \frac{\delta}{2} \mathbb{V}_{p(x \mid \xi)} \! \left[ \mathbb{E}_{p(\theta \mid x, \xi)}\left[L(\theta, x, \xi)\right] \right] + \mathcal{O}(\delta^{2}),
\end{equation}
where we used $\kld{p(\theta \mid x, \xi)}{p(\theta)} = \mathbb{E}_{p(\theta \mid x, \xi)}\left[L(\theta, x, \xi)\right]$. Finally, by the law of total variance applied to $L$ under the joint $p(\theta, x \mid \xi) = p(x \mid \xi) \, p(\theta \mid x, \xi)$:
\begin{equation}
    \mathbb{V}_{p(\theta, x \mid \xi)} \! \left[L(\theta, x, \xi)\right] = \mathbb{E}_{p(x \mid \xi)} \! \left[ \mathbb{V}_{p(\theta \mid x, \xi)} \! \left[L(\theta, x, \xi)\right] \right] + \mathbb{V}_{p(x \mid \xi)} \! \left[ \mathbb{E}_{p(\theta \mid x, \xi)}\left[L(\theta, x, \xi)\right] \right].
\end{equation}
Substituting and restoring $\delta = \alpha - 1$ yields the expansion:
\begin{equation}
    I_{\alpha}^{\timesit{S}}(\theta; x)(\xi) = I(\xi) + \frac{\alpha - 1}{2} \, \mathbb{V}_{p(\theta, x \mid \xi)} \! \left[ \log \frac{p(\theta \mid x, \xi)}{p(\theta)} \right] + \mathcal{O}\left((\alpha - 1)^{2}\right).
\end{equation}

\qed

\subsection{Proof of Proposition~\ref{prop:uninformative}}

For the proof that follows, we require that the nominal model carries finite information content.

\begin{assumption}[Finite expected surprise]
    \label{asm:finite-surprise}
    The integral
    \begin{align}
        \int p(\theta) \max \big\{1, p(x \mid \theta, \xi) \big\} \, \abs{\log p(x \mid \theta, \xi)} \dif \theta < \infty
    \end{align}
    is well-defined and finite for all $(\xi, x)$. Additionally, the geometric average log-likelihood
    \begin{equation}\label{eq:H0}
        H_{0}(\xi) \coloneq \int \exp \left\{\int p(\theta)  \log p(x \mid \theta, \xi) \dif \theta \right\} \dif x
    \end{equation}
    is uniformly bounded away from $0$: there exists a constant $C_{0} > 0$, such that for all $\xi \in \Xi$, it holds: $1 \geq H_{0}(\xi) > C_{0}$.
\end{assumption}

The behavior of the robust expected information gain is driven by the multiplier $\alpha/(\alpha-1)$, so it is sufficient to verify that the remaining integral term remains bounded. For this purpose, we identify $H_{0}(\xi)$, defined in~\eqref{eq:H0}, as a uniform lower bound. This bound permits a direct application of Assumption~\ref{asm:finite-surprise}, from which the desired convergence result follows.

Let us start by defining the pointwise quantity:
\begin{equation}
    I(x, \xi; \alpha) \coloneq \log \left\{ \left[ \int p(\theta) \, p(x \mid \theta, \xi)^{\alpha} \dif \theta \right]^{1/\alpha} \right\}.
\end{equation}
Note that $I(x, \xi; \cdot)$ is non-decreasing in $\alpha$. We determine its limit as $\alpha \to 0$ by writing it as a ratio $f(\alpha)/\alpha$:
\begin{equation}
    I(x, \xi; \alpha) = \frac{1}{\alpha} \log \int p(\theta) \, p(x \mid \theta, \xi)^{\alpha} \dif \theta.
\end{equation}
This presents an indeterminate limit $0/0$. Applying L'H\^opital's rule, the limit corresponds to $\lim_{\alpha \to 0^+} f'(\alpha)$, provided it exists. The derivative is:
\begin{equation}
    f'(\alpha) = \frac{\displaystyle \frac{\partial}{\partial \alpha}\int p(\theta) \, p(x \mid \theta, \xi)^{\alpha} \dif \theta}{\displaystyle \int p(\theta) \, p(x \mid \theta, \xi)^{\alpha} \dif \theta}.
\end{equation}
We analyze the denominator and numerator of $f'(\alpha)$ separately. For the denominator, Lebesgue's dominated convergence theorem, applied with the bound $p(x \mid \theta, \xi)^{\alpha} \leq \max \{1, p(x \mid \theta, \xi)\}$, ensures convergence to $1$. For the numerator, the same theorem, combined with Lebesgue's differentiation theorem applied to the derivative of the integrand:
\begin{equation}
    \abs{p(x \mid \theta, \xi)^{\alpha} \log p(x \mid \theta, \xi)} \leq \max \big\{1, p(x \mid \theta, \xi) \big\} \, \abs{\log p(x \mid \theta, \xi)},
\end{equation}
which is finite by Assumption~\ref{asm:finite-surprise}, ensures that the numerator converges to
\begin{equation}
    \int p(\theta) \log p(x \mid \theta, \xi) \dif \theta.
\end{equation}
Combining these results yields the pointwise bounds for all $\alpha \in (0, 1)$, $x \in \mathcal{X}$, and $\xi \in \Xi$:
\begin{align}
    \int p(\theta) \log p(x \mid \theta, \xi) \dif \theta \leq I(x, \xi; \alpha) \leq I(x, \xi; 1) = \log \int p(\theta) \, p(x \mid \theta, \xi) \dif \theta.
\end{align}
Exponentiating and integrating with respect to $x$ provides
\begin{align}
    H_{0}(\xi) \leq \int \exp \big\{I(x, \xi; \alpha)\big\} \dif x \leq \iint p(\theta) \, p(x \mid \theta, \xi) \dif \theta \dif x = 1.
\end{align}
By Assumption~\ref{asm:finite-surprise}, we have $H_{0}(\xi) \geq C_{0} > 0$. Consequently, the log-integral term is strictly bounded:
\begin{align}
    -\infty < \log C_{0} \leq \log \int \left[ \int p(\theta) \, p(x \mid \theta, \xi)^{\alpha} \dif \theta \right]^{1 / \alpha} \dif x \leq 0.
\end{align}
This bound holds uniformly in $\xi$ and $\alpha$. Since the multiplier $\alpha / (\alpha - 1)$ converges to zero, the robust expected information gain vanishes as $\alpha \to 0$, completing the proof.
\qed

\subsection{Proof of Proposition~\ref{prop:robust-eig-bias}}
\label{app:proof-robust-eig-bias}

This assumption, used in Proposition~\ref{prop:robust-eig-bias} and Lemmata~\ref{lem:psi-bias} and~\ref{lem:psi-variance}, controls the design objective behavior.

\begin{assumption}[Regularity and boundedness]
    \label{asm:robust-eig-regularity}
    We assume the parameter $\alpha \in (0, 1)$ and that the design set $\Xi$ has a finite diameter $D_{\Xi}$. The ratios $w(x, \theta, \cdot)^{\alpha}$ are $L_{w}$-Lipschitz, bounded from below by $c_{w} > 0$ and from above by $C_{w}$, and have bounded variance $\smash{\sigma_{w}^{2} \coloneq \sup_{\xi, x} \mathbb{V}_{\!\theta}[w^{\alpha}]}$. On the interval $[c_{w}, C_{w}]$, the function $h$ is $L_{h}$-Lipschitz, bounded above by $C_{h}$, and has bounded variance $\smash{\sigma_{h}^{2} \coloneq \sup_{\xi} \mathbb{V}_{\!x \mid \xi}[(\mathbb{E}_{\theta}[w^{\alpha}])^{1/\alpha}]}$. Consequently, the functions $g$ and $\tilde{g}$ take values in $[\tau, C_{h}]$ with $\smash{\tau \coloneq c_{w}^{1/\alpha} > 0}$, on which the function $f$ is $L_{f}$-Lipschitz.
\end{assumption}

\begin{lemma}[Bias bound for $\tilde{g}(\xi)$, \citealp{hu2020sample}]
    \label{lem:psi-bias}
    Under Assumption~\ref{asm:robust-eig-regularity}, for any design $\xi \in \Xi$, the bias of the estimator $\tilde{g}(\xi)$ is bounded by:
    \begin{equation}
        \big| \mathbb{E} \big[ \tilde{g}(\xi) \big] - g(\xi) \big| \le \frac{L_{h} \sigma_{w}}{\sqrt{M}}.
    \end{equation}
\end{lemma}

\begin{lemma}[Variance bound for $\tilde{g}(\xi)$, \citealp{hu2020sample}]
    \label{lem:psi-variance}
    Under Assumption~\ref{asm:robust-eig-regularity}, for any design $\xi \in \Xi$, the variance of the estimator $\tilde{g}(\xi)$ is bounded by:
    \begin{equation}
        \mathbb{V} \big[ \tilde{g}(\xi) \big] \le \frac{\sigma_{h}^{2}}{N} + \frac{4 C_{h} L_{h} \sigma_{w}}{N\sqrt{M}}.
    \end{equation}
\end{lemma}

Recall that $\tilde{I}^{\timesit{S}}_{\alpha}(\xi) = f(\tilde{g}(\xi))$, where the functions $f$ and $g$ are defined in Section~\ref{sec:nested-mc}. We decompose the absolute error by adding and subtracting the term $f(\mathbb{E} [\tilde{g}])$ and then applying the triangle inequality:
\begin{equation}
    \big| \mathbb{E} [f(\tilde{g})]  - f(g) \big| \leq \big| \mathbb{E} [f(\tilde{g})] - f(\mathbb{E} [\tilde{g}]) \big| + \big| f(\mathbb{E} [\tilde{g}])  - f(g) \big|.
\end{equation}
We bound the first term by applying Jensen's inequality and invoking the Lipschitz continuity of $f$:
\begin{align}
    \big| \mathbb{E} [f(\tilde{g})] - f(\mathbb{E} [\tilde{g}]) \big| & = \big| \mathbb{E} [ f(\tilde{g}) - f(\mathbb{E} [\tilde{g}]) ] \big| \\
    & \leq \mathbb{E} [| f(\tilde{g}) - f(\mathbb{E} [\tilde{g}]) |] \leq L_{f} \, \mathbb{E} [| \tilde{g} - \mathbb{E} [\tilde{g}] |],
\end{align}
then we apply the Cauchy--Schwarz inequality, $\mathbb{E}[|X|] \leq \sqrt{\mathbb{E}[X^{2}]}$, and obtain:
\begin{equation}
    L_{f} \, \mathbb{E} \left[| \tilde{g} - \mathbb{E} [\tilde{g}] |\right] \leq L_{f} \sqrt{\mathbb{V}[\tilde{g}]},
\end{equation}
where the variance $\mathbb{V}[\tilde{g}]$ is controlled by Lemma~\ref{lem:psi-variance}. The second term can be bounded directly using Lipschitz continuity:
\begin{equation}
    \big| f(\mathbb{E} [\tilde{g}])  - f(g) \big| \leq L_{f} \big| \mathbb{E} [ \tilde{g}] - g \big|,
\end{equation}
where $\big| \mathbb{E} [ \tilde{g}] - g \big|$ is bounded by Lemma~\ref{lem:psi-bias}. Finally, combining these bounds yields the stated result.
\qed

\subsection{Proof of Proposition~\ref{prop:estimator-convergence}}
\label{app:proof-prop-estimator-convergence}

We briefly state a two-sided version of Hoeffding's inequality for bounded random variables that aids in the proof.

\begin{lemma}[Two-sided Hoeffding inequality for bounded random variables, \citealp{vershynin2018high}]
    \label{lem:two-sided-hoeffding}
    Let $X_{1}, \dots, X_{N}$ be independent random variables such that $X_{i} \in [a_{i}, b_{i}]$ for every $i \in \{1, \dots, N\}$. Then, for any $t > 0$, we have
    \begin{equation}
        \mathbb{P} \Big( \Big| \sum_{i=1}^{N} \big( X_{i} - \mathbb{E}[X_{i}] \big) \Big| \geq t \Big) \le 2 \exp\left\{ -\frac{2 t^{2}}{ \sum_{i=1}^{N} (b_{i} - a_{i})^{2}} \right\}.
    \end{equation}
\end{lemma}

Recall that $\tilde{I}^{\timesit{S}}_{\alpha}(\xi) = f(\tilde{g}(\xi))$, where the functions $f$ and $g$ are defined in Section~\ref{sec:nested-mc}. By invoking the Lipschitz continuity of $f$ and applying the triangle inequality, we decompose the total error into a \emph{stochastic} and \emph{deterministic} component:
\begin{align}
    \big| \tilde{I}^{\timesit{S}}_{\alpha} - I^{\timesit{S}}_{\alpha} \big| = \big| f(\tilde{g}) - f(g) \big| & \leq L_{f} \big| \tilde{g} - g \big| \\
    & \leq \underbrace{L_{f}  \big| \tilde{g} - \mathbb{E}[ \tilde{g} ] \big|}_{\text{stochastic}} + \underbrace{L_{f} \big| \mathbb{E}[ \tilde{g} ] - g \big|}_{\text{deterministic}}.
\end{align}
According to Lemma~\ref{lem:psi-bias}, the deterministic term $\big| \mathbb{E}[\tilde{g}] - g \big|$ is upper bounded by $L_{h} \sigma_{w} /\sqrt{M}$. To ensure the total error stays within tolerance $t$, we constrain this error term to consume at most half of the total budget $t/2$ by choosing $M$ such that:
\begin{equation}
    \frac{L_{h} \sigma_{w}}{\sqrt{M}} \leq \frac{t}{2 L_{f}} \quad \implies M \ge \frac{4 L_{f}^{2} L_{h}^{2} \sigma_{w}^{2}}{t^{2}}.
\end{equation}
Consequently, for the total error to exceed $t$, the stochastic term $\big| \tilde{g} - \mathbb{E}[\tilde{g}] \big|$ must account for the remaining difference:
\begin{equation}
    L_{f} \big| \tilde{g} - \mathbb{E}[\tilde{g}] \big| \ge t - \frac{t}{2} = \frac{t}{2} \quad \implies \quad \big| \tilde{g} - \mathbb{E}[\tilde{g}] \big| \ge \frac{t}{2 L_{f}}.
\end{equation}
Next, we bound the probability of this deviation. Recall that $\tilde{g}$ is the empirical average of $N$ independent random variables $\tilde{h}^{(i)}$, defined as:
\begin{equation}
    \tilde{h}^{(i)} \coloneq h\Big( \frac{1}{M} \sum_{j=1}^{M} w(x^{(i)}, \theta^{(i,j)}, \xi)^{\alpha} \Big).
\end{equation}
By Assumption~\ref{asm:robust-eig-regularity}, the function $h$ is bounded by $C_{h}$, implying that each random variable $\tilde{h}^{(i)} \in [0, C_{h}]$. We can therefore apply Lemma~\ref{lem:two-sided-hoeffding} to the sample average $\tilde{g} = \frac{1}{N} \sum_{i=1}^{N} \tilde{h}^{(i)}$:
\begin{equation}
    \mathbb{P} \Big( \big| \tilde{g} - \mathbb{E}[\tilde{g}] \big| \geq \frac{t}{2 L_{f}} \Big) \leq 2 \exp \left\{ - \frac{N t^{2}}{2 L_{f}^{2} C_{h}^{2}} \right\}.
\end{equation}
Combining the bounds on the stochastic and deterministic components, we obtain the final concentration inequality:
\begin{equation}
    \mathbb{P} \left( \big| \tilde{I}_{\alpha}^{\timesit{S}} - I_{\alpha}^{\timesit{S}} \big| \geq t \right) \le 2 \exp\left\{ -\frac{N t^{2}}{2 L_{f}^{2} C_{h}^{2}} \right\}.
\end{equation}
\qed

\subsection{Proof of Proposition~\ref{prop:pac-bayes-bound}}
\label{app:proof-pac-bayes-bound}

This proof follows standard PAC-Bayes proof techniques presented by~\citet{alquier2024pacbayes}. We begin by stating three useful lemmata that will facilitate the derivation.

\begin{lemma}[Hoeffding's lemma, \citealp{vershynin2018high}]
    \label{lem:hoeffding-lemma}
    Let $X$ be a random variable such that $X \in [a, b]$ almost surely. Then, for any $s \in \mathbb{R}$, it holds that:
    \begin{equation}
        \mathbb{E} \left[ \exp\{ s \, (X - \mathbb{E}[X]) \} \right] \leq \exp \left\{ \frac{s^{2} (b - a)^{2}}{8} \right\}.
    \end{equation}
\end{lemma}

\begin{lemma}[Markov inequality, \citealp{vershynin2018high}]
    \label{lem:markov-inequality}
    For any non-negative random variable $X$ and $a > 0$, we have:
    \begin{equation}
        \mathbb{P}(X \geq a) \leq \frac{\mathbb{E}[X]}{a} \quad \Longleftrightarrow \quad \mathbb{P}(X \geq k \, \mathbb{E}[X]) \leq \frac{1}{k}, \,\, \text{where} \,\, k = \frac{a}{\mathbb{E}[X]}.
    \end{equation}
\end{lemma}

\begin{lemma}[Donsker--Varadhan variational formula, \citealp{donsker1975asymptotic}]
    \label{lem:donsker-varadhan}
    For any measurable, bounded function $f: \Xi \mapsto \mathbb{R}$, and any probability measure $\pi_{0} \in \Pi$, we have:
    \begin{equation}
        \log \mathbb{E}_{\pi_{0}} \left[ e^{f(\xi)} \right]= \sup_{\pi \in \Pi} \left\{ \mathbb{E}_{\pi} \left[ f(\xi) \right] - \kld{\pi}{\pi_{0}} \right\}.
    \end{equation}
\end{lemma}

To prove the proposition, we first analyze the concentration properties of the estimation error $\tilde{I}_{\alpha}^{\timesit{S}}(\xi) - I_{\alpha}^{\timesit{S}}(\xi)$ for a fixed design $\xi \in \Xi$. Let $\bar{g}(\xi) \coloneq \mathbb{E}[\tilde{g}(\xi)]$ denote the mean of the empirical estimator, where the expectation is evaluated under the sample distribution of the naive nested Monte Carlo estimator. By the Lipschitz continuity of $f$ under Assumption~\ref{asm:robust-eig-regularity} and the bias bound of Lemma~\ref{lem:psi-bias}, the estimation error decomposes into two terms:
\begin{equation}
    \tilde{I}_{\alpha}^{\timesit{S}}(\xi) - I_{\alpha}^{\timesit{S}}(\xi) = \big[ f(\tilde{g}) - f(\bar{g}) \big] + \big[ f(\bar{g}) - f(g) \big] \leq L_{f} \, \big| \tilde{g} - \bar{g} \big| + \frac{L_{f} L_{h} \sigma_{w}}{\sqrt{M}}.
\end{equation}
Recall from the proof of Proposition~\ref{prop:estimator-convergence} that $\tilde{g}$ is the empirical average of $N$ independent random variables $\tilde{h}^{(i)} \in [0, C_{h}]$. Using Lemma~\ref{lem:hoeffding-lemma} on each $\tilde{h}^{(i)}$, for any $s \in \mathbb{R}$, we have:
\begin{equation}
    \mathbb{E} \Big[ \exp \big\{ s \, (\tilde{g} - \bar{g}) \big\} \Big] = \prod_{i=1}^{N} \mathbb{E} \Big[ \exp \Big\{ \frac{s}{N} \big( \tilde{h}^{(i)} - \mathbb{E}[\tilde{h}^{(i)}] \big) \Big\} \Big] \leq \exp \left\{ \frac{s^{2} C_{h}^{2}}{8N} \right\}.
\end{equation}
Applying this for $s$ and $-s$, and using the inequality $e^{s |z|} \leq e^{s z} + e^{- s z}$, we obtain for any $s > 0$:
\begin{equation}
    \mathbb{E} \Big[ \exp \big\{ s \, |\tilde{g} - \bar{g}| \big\} \Big] \leq 2 \exp \left\{ \frac{s^{2} C_{h}^{2}}{8N} \right\}.
\end{equation}
Combining both results with $s = \lambda L_{f}$, for any $\lambda > 0$, the moment generating function of the estimation error is bounded as:
\begin{equation}
    \mathbb{E} \Big[ \exp \Big\{ \lambda \Big( \tilde{I}_{\alpha}^{\timesit{S}}(\xi) - I_{\alpha}^{\timesit{S}}(\xi) \Big) \Big\} \Big] \leq 2 \exp \left\{ \frac{\lambda L_{f} L_{h} \sigma_{w}}{\sqrt{M}} + \frac{\lambda^{2} L_f^{2} C_{h}^{2}}{8N} \right\}.
\end{equation}
We integrate this bound under the policy $\pi_{0}$ and swap the expectations yielding:
\begin{equation}
    \mathbb{E} \Big[ \mathbb{E}_{\pi_{0}} \Big[ \exp \Big\{ \lambda \Big( \tilde{I}_{\alpha}^{\timesit{S}}(\xi) - I_{\alpha}^{\timesit{S}}(\xi) \Big) \Big\} \Big] \Big] \leq 2 \exp \left\{ \frac{\lambda L_{f} L_{h} \sigma_{w}}{\sqrt{M}} + \frac{\lambda^{2} L_f^{2} C_{h}^{2}}{8N} \right\}.
\end{equation}
Next, we apply Lemma~\ref{lem:donsker-varadhan} to the inner expectation and obtain the following:
\begin{equation}
    \mathbb{E} \Big[ \exp \Big\{ \sup_{\pi \in \Pi} \Big\{ \lambda \, \mathbb{E}_{\pi} \left[ \tilde{I}_{\alpha}^{\timesit{S}}(\xi) - I_{\alpha}^{\timesit{S}}(\xi) \right] - \kld{\pi}{\pi_{0}} \Big\} \Big\} \Big] \leq 2 \exp \left\{ \frac{\lambda L_{f} L_{h} \sigma_{w}}{\sqrt{M}} + \frac{\lambda^{2} L_f^{2} C_{h}^{2}}{8N} \right\}.
\end{equation}
This gives us a bound on the expectation of the form: $\mathbb{E}[X] \leq b$, where $X$ is the exponential term inside the outer expectation. To convert this expectation bound back to a probability bound, we apply Lemma~\ref{lem:markov-inequality} on $X$ and assume $k = 1 / \delta$ for $\delta \in (0, 1)$:
\begin{equation}
    \mathbb{P} \Big( \exp \Big\{ \sup_{\pi \in \Pi} \Big\{ \lambda \, \mathbb{E}_{\pi} \big[ \tilde{I}_{\alpha}^{\timesit{S}}(\xi) - I_{\alpha}^{\timesit{S}}(\xi) \big] - \kld{\pi}{\pi_{0}} \Big\} \Big\} \geq \frac{2}{\delta} \exp \left\{ \frac{\lambda L_{f} L_{h} \sigma_{w}}{\sqrt{M}} + \frac{\lambda^{2} L_f^{2} C_{h}^{2}}{8N} \right\} \Big) \leq \delta.
\end{equation}
Taking the logarithm of both sides and rearranging the terms leads to:
\begin{equation}
    \mathbb{P} \Big( \exists \pi \in \Pi, \, \mathbb{E}_{\pi} \big[ \tilde{I}_{\alpha}^{\timesit{S}}(\xi) \big] - \frac{\lambda L_f^{2} C_{h}^{2}}{8N} - \frac{L_{f} L_{h} \sigma_{w}}{\sqrt{M}} - \frac{\kld{\pi}{\pi_{0}} + \log(2 / \delta)}{\lambda} \geq \mathbb{E}_{\pi} \big[ I_{\alpha}^{\timesit{S}}(\xi) \big] \Big) \leq \delta.
\end{equation}
Finally, the condition $M \geq 2 N L_{h}^{2} \sigma_{w}^{2} / (C_{h}^{2} \log (2 / \delta))$ ensures that the bias term is dominated by the concentration terms. It implies $L_{h} \sigma_{w} / \sqrt{M} \leq C_{h} \sqrt{\log(2 / \delta) / (2N)}$, so that by the inequality $2 \sqrt{ab} \leq \lambda a + b / \lambda$, valid for any $a, b \geq 0$ and $\lambda > 0$, we have:
\begin{equation}
    \frac{L_{f} L_{h} \sigma_{w}}{\sqrt{M}} \leq L_{f} C_{h} \sqrt{\frac{\log(2 / \delta)}{2N}} \leq \frac{\lambda L_f^{2} C_{h}^{2}}{8N} + \frac{\log(2 / \delta)}{\lambda}.
\end{equation}
Substituting this bound and taking the complement of the previous result, we retrieve a lower bound on $\mathbb{E}_{\pi} \left[ I_{\alpha}^{\timesit{S}}(\xi) \right]$:
\begin{equation}
    \mathbb{P} \Big( \forall \pi \in \Pi, \, \mathbb{E}_{\pi} \big[ \tilde{I}_{\alpha}^{\timesit{S}}(\xi) \big] - \frac{\lambda L_f^{2} C_{h}^{2}}{4N} - \frac{\kld{\pi}{\pi_{0}} + 2 \log(2 / \delta)}{\lambda} \leq \mathbb{E}_{\pi} \big[ I_{\alpha}^{\timesit{S}}(\xi) \big] \Big) \geq 1 - \delta,
\end{equation}
that holds with probability at least $1 - \delta$.
\qed

\section{Implementation of the Nested Monte Carlo Estimator}
\label{app:nested-estimator}

To ensure numerical stability and prevent arithmetic underflow, we perform all computations strictly in the logarithmic domain. We define the log-likelihood function as $\ell(\theta, x) \coloneq \log p(x \mid \theta, \xi)$ and the log-sum-exp operator as $\mathrm{LSE}(y_{1:K}) \coloneq \log \sum_{k=1}^{K} \exp(y_{k})$. The estimation procedure proceeds in four steps:
\begin{itemize}
    \item We first generate an outer batch of $N$ independent $(\theta, x)$ pairs from the joint generative model. For each outer sample with index $i \in \{1, \dots, N\}$, we independently draw a set of $M$ inner parameters from the prior:
    \begin{equation}
        \{ (\theta^{(i)}, x^{(i)}) \}_{i=1}^{N} \sim p(\theta) \, p(x \mid \theta, \xi), \qquad
        \{ \theta^{(i,j)} \}_{j=1}^{M} \sim p(\theta) \quad \text{for each } i \in \{1, \dots, N\}.
    \end{equation}

    \item As per Corollary~\ref{cor:nested-representation}, we now need to compute $w(x^{(i)}, \theta^{(i, j)}, \xi) = p(x^{(i)} \mid \theta^{(i, j)}, \xi) / p(x^{(i)} \mid \xi) \eqcolon w^{(i, j)}$. When the marginal likelihood $p(x \mid \xi)$ is available in closed form, as in the conjugate models of Appendices~\ref{app:linear-regression} and~\ref{app:ab-testing}, we evaluate it directly:
    \begin{equation}
        \log w^{(i,j)} = \ell(\theta^{(i,j)}, x^{(i)}) - \log p(x^{(i)} \mid \xi).
    \end{equation}
    When $p(x \mid \xi)$ is intractable, we instead use the contrastive estimator $\tilde{w}$ of Definition~\ref{def:robust-eig-estimator}.

    \item We next estimate the inner expectation term $\varphi(x, \xi) = \mathbb{E}_{p(\theta)}[w(x, \theta, \xi)^{\alpha}]$ by averaging the powered weights. In the log-domain, this corresponds to a second application of the LSE operator over the $M$ inner samples:
    \begin{equation}
        \log \varphi(x^{(i)}, \xi) = \mathrm{LSE}\left( \left\{ \alpha \log w^{(i,j)} \right\}_{j=1}^{M} \right) - \log M.
    \end{equation}

    \item Finally, we compute the robust expected information gain by aggregating the outer samples. The estimator $\tilde{I}^{\timesit{S}}_{\alpha}(\xi)$ approximates the logarithm of the expected exponential gain:
    \begin{equation}
        \tilde{I}^{\timesit{S}}_{\alpha}(\xi) = \frac{\alpha}{\alpha - 1} \left[ \mathrm{LSE}\left( \left\{ \alpha^{-1} \log \varphi(x^{(i)}, \xi) \right\}_{i=1}^{N} \right) - \log N \right].
    \end{equation}
\end{itemize}
This procedure yields a consistent, albeit biased, estimate of the robust objective.

\section{Theoretical Comparison of PAC-Bayesian and Deterministic Policies}
\label{app:pac-bayes-advantage}

The theoretical advantages of PAC-Bayesian bounds are comprehensively summarized in \citet{alquier2024pacbayes}.
In the context of bandit problems, which are highly relevant to experimental design, stochastic policies derived from the PAC-Bayesian formulation facilitate efficient exploration while streamlining the theoretical analysis~\citep{flynn2023pac}.
This section aims to compare the concentration properties of deterministic and stochastic policies, illustrating that PAC-Bayesian bounds can yield a sharper upper estimate of the minimum of the information gain under certain conditions.
These results further reinforce the preference for stochastic policies, as they do not necessarily compromise the concentration rate in spite of the induced randomness, while leveraging the theoretical versatility of the PAC-Bayesian framework.

Consider the minimization of an objective function $J(\xi)$ over the design space $\Xi$, where the analysis applies symmetrically to maximization.
To streamline the illustration, we adopt a simplified setup where the objective function $J(\xi)$ is defined by $J(\xi) = \mathbb{E}_{z \sim p}[ h(\xi, z) ]$, where $h(\xi, \cdot)$ is an integrand with a latent variable $z$ defined on a compact support $\mathcal{Z} \subset \mathbb{R}^k$ and $p$ is the distribution of the latent variable $z$.
We approximate the objective function $J(\xi)$ using the empirical Monte Carlo estimator $\tilde{J}(\xi)$ defined by $N$ i.i.d.~realizations $z_1, \dots, z_N$ drawn from $p$:
\begin{align}
	J(\xi) = \mathbb{E}_{z \sim p}[ h(\xi, z) ] \quad \approx \quad \tilde{J}(\xi) = \frac{1}{N} \sum_{i=1}^{N} h(\xi, z_i) .
\end{align}
While we seek to minimize $J(\xi)$, the exact objective  $J(\xi)$ is inaccessible in practice.
Consequently, the deterministic policy minimizes the Monte Carlo estimator $\tilde{J}(\xi)$ as a proxy.
We define
\begin{align}
	\xi_{\text{min}} \in \arg\min_{\xi} J(\xi) \quad \text{and} \quad \tilde{\xi}_{\text{min}} \in \arg\min_{\xi} \tilde{J}(\xi) .
\end{align}
One justification for employing the deterministic policy is that the minimum of the Monte Carlo estimator $\tilde{J}(\xi)$ provides an upper estimate of the global minimum $J(\xi_{\text{min}})$.

To derive the upper estimate, we introduce a standard measure from statistical learning theory, called the Rademacher complexity.
Intuitively, the Rademacher complexity measures the expressive power of a function family relative to samples of size $N$.
Let $\mathcal{R}_N$ denote the Rademacher complexity of the integrand family $\{ h(\cdot, z) \mid z \in \mathcal{Z} \}$ computed over $N$ random samples $\{ z_i \}_{i=1}^{N}$ following $p$.
For the purposes of this illustration, the formal definition of Rademacher complexity is omitted for brevity.
We refer the reader to \citet{wainwright2019high} for a rigorous treatment of Rademacher complexity and its upper bounds for common function classes.
The next assumption is used throughout this section.

\begin{assumption}
	The integrand $h(\xi, z)$ is non-negative, continuous, and uniformly bounded by a constant $b$.
	The objective function $J(\xi)$ is uniquely minimized over $\Xi$.
\end{assumption}

The following upper estimate of the global minimum $J(\xi_{\text{min}})$ is a standard result derived from empirical process theory, particularly in the study of empirical risk minimization.

\begin{proposition} \label{prop:empirical_minimiser_concentration}
	It holds for any $N > 0$ and $\delta > 0$ that
	\begin{align}
		\mathbb{P}\left( J(\xi_{\text{min}}) \le \min_{\xi} \tilde{J}(\xi) + 2 \mathcal{R}_N + b \sqrt{ \frac{ 2 \log( 1 / \delta ) }{ N } } \right) \ge 1 - \delta ,
	\end{align}
	where the minimum is attained at $\xi = \tilde{\xi}_{\text{min}}$.
\end{proposition}

\begin{proof}
	We have $J(\xi_{\text{min}}) \le J(\xi)$ for any point $\xi$, since $\xi_{\text{min}}$ is the minimum.
	This, in turn, yields that
	\begin{align}
		J(\xi_{\text{min}}) - \tilde{J}(\tilde{\xi}_{\text{min}}) \le J(\tilde{\xi}_{\text{min}}) - \tilde{J}(\tilde{\xi}_{\text{min}}) \le \sup_{\xi} | J(\xi) - \tilde{J}(\xi) | .
	\end{align}
	The following concentration inequality is a direct result of Theorem 4.10 of \citet{wainwright2019high}:
	\begin{align}
		\mathbb{P}\left( \sup_{\xi} | J(\xi) - \tilde{J}(\xi) | \le 2 \mathcal{R}_N + b \sqrt{ \frac{ 2 \log( 1 / \delta ) }{ N } } \right) \ge 1 - \delta .
	\end{align}
	Rearranging the term completes the proof.
\end{proof}

By definition, the empirical minimizer $\tilde{\xi}_{\text{min}}$ yields the sharpest possible bound of this form, as any deviation from the empirical minimum necessarily increases the value of the upper estimate.

In contrast to the above concentration inequality, the PAC-Bayesian framework provides a flexible upper estimate that bypasses the need for global complexity measures such as the Rademacher complexity.
Given a prior $\pi_0$ over the design space $\Xi$, we define a stochastic (Gibbs) policy $\pi^\star$ as
\begin{align}
	\pi^\star(\xi) \propto \exp\left\{ - \lambda \tilde{J}(\xi) \right\} \pi_0(\xi) = \arg \min_{\rho}\left\{ \mathbb{E}_{\rho}[ \tilde{J}(\xi) ] + \frac{ \kld{\rho}{\pi_{0}} }{\lambda} \right\} ,
\end{align}
where $\lambda > 0$ is the precision hyperparameter that controls the policy dispersion.
A well-established bound by \citet{catoni2007pac} provides the following upper estimate under the stochastic policy $\pi^\star$.

\begin{proposition}
	It holds for any $N > 0$ and $\delta > 0$ that
	\begin{align}
		\mathbb{P}\left( J(\xi_{\text{min}}) \le \min_{\rho}\left\{ \mathbb{E}_{\rho}[ \tilde{J}(\xi) ] + \frac{ \kld{\rho}{\pi_{0}} }{\lambda} \right\} + \frac{\lambda b^2}{8 N} + \frac{ \log( 1 / \delta ) }{\lambda} \right) \ge 1 - \delta ,
	\end{align}
	where the minimum is attained at $\rho = \pi^\star$.
\end{proposition}

\begin{proof}
	We have $J(\xi_{\text{min}}) \le \mathbb{E}_{\rho}[ J(\xi) ]$ for an arbitrary distribution $\rho$, since $\xi_{\text{min}}$ is the minimum and $J$ is non-negative.
	Applying Catoni's bound~\citep{catoni2007pac} to $\mathbb{E}_{\rho}[ J(\xi) ]$ completes the proof.
\end{proof}

Employing the stochastic policy offers two primary benefits: (i) it permits a more tractable and flexible theoretical treatment by circumventing dependency on global complexity measures, and (ii) it promotes exploration across the design space $\Xi$, as discussed in \citet{flynn2023pac}.
Furthermore, it may yield a sharper upper estimate of the global minimum $J(\xi_{\text{min}})$ than the standard bound.
Finally, the following proposition demonstrates that the PAC-Bayesian bound can be strictly sharper than the other, provided that the level $\delta$ and the precision parameter $\lambda$ are appropriately chosen.

\begin{proposition}
	Choose a level $\delta$ such that $\delta \le e^{-8N}$ for a given value of the sample number $N$.
	Choose a precision parameter $\lambda$ such that $\lambda \in ( (3 - \sqrt{5}) \sqrt{ 2 (N / b^2) \log( 1 / \delta ) } , ( 3 + \sqrt{5} ) \sqrt{ 2 (N / b^2) \log( 1 / \delta ) } )$.
	Then, it holds that
	\begin{align}
		\min_{\rho}\left\{ \mathbb{E}_{\rho}[ \tilde{J}(\xi) ] + \frac{ \kld{\rho}{\pi_{0}} }{\lambda} \right\} + \frac{\lambda b^2}{8 N} + \frac{ \log( 1 / \delta ) }{\lambda} < \min_{\xi} \tilde{J}(\xi) + 2 \mathcal{R}_N + b \sqrt{ \frac{ 2 \log( 1 / \delta ) }{ N } } .
	\end{align}
\end{proposition}

\begin{proof}
	By the Donsker--Varadhan variational formula from Lemma~\ref{lem:donsker-varadhan}, we have
	\begin{align}
		\min_{\rho}\left\{ \mathbb{E}_{\rho}[ \tilde{J}(\xi) ] + \frac{ \kld{\rho}{\pi_{0}} }{\lambda} \right\} = - \frac{1}{\lambda} \log\left( \mathbb{E}_{\pi_0}[ \exp\{ - \lambda \tilde{J}(\xi) \} ] \right) .
	\end{align}
	It follows from the upper bound of $\delta$ that $b \le b ( 1 / 4 ) \sqrt{ 2 \log(1 / \delta) / N }$.
	Since $\tilde{J}$ is non-negative and bounded by $b$, we have $\tilde{J}(\xi) \le b \le \tilde{J}(\tilde{\xi}_{\text{min}}) + 2 \mathcal{R}_N + (1 / 4 ) b \sqrt{ 2 \log(1 / \delta) / N }$.
	Substituting this into the Donsker--Varadhan formula gives that
	\begin{align}
		\min_{\rho}\left\{ \mathbb{E}_{\rho}[ \tilde{J}(\xi) ] + \frac{ \kld{\rho}{\pi_{0}} }{\lambda} \right\} \le \tilde{J}(\tilde{\xi}_{\text{min}}) + 2 \mathcal{R}_N + \frac{b}{4} \sqrt{ \frac{ 2 \log( 1 / \delta ) }{ N } } .
	\end{align}
	Thus, to establish the main result, it suffices to show that
	\begin{align}
		\frac{\lambda b^2}{8 N} + \frac{ \log( 1 / \delta ) }{\lambda} \le \frac{3 b}{4} \sqrt{ \frac{ 2 \log( 1 / \delta ) }{ N } } \quad \Longleftrightarrow \quad \frac{b^2}{8 N} \lambda^2 - \frac{3 b}{4} \sqrt{ \frac{ 2 \log( 1 / \delta ) }{ N } } \lambda + \log \frac{1}{\delta} \le 0 .
	\end{align}
	Solving the quadratic inequality on the right hand side completes the proof.
\end{proof}

This comparison is not intended to suggest that the PAC-Bayesian framework generally yields a sharper bound than the classical counterparts; their performance remains contingent on the choice of prior and the concentration of the empirical objective.
Nevertheless, it illustrates that the PAC-Bayesian framework can provide the latitude to enhance concentration properties, while offering flexibility in the theoretical analysis.
In this illustration, for some confidence level $\delta$, there exists a regime of the precision parameter $\lambda$ where the PAC-Bayesian bound is strictly sharper than the other bound.
Such theoretical insights may further facilitate a principled selection of the precision parameter $\lambda$ in practice.

\section{Choosing the PAC-Bayes Precision $\lambda$}
\label{app:pac-bayes-precision}

{\color{black}
The parameter $\lambda$ controls the inverse temperature of the PAC-Bayes policy and the tightness of the finite-sample bound. For a fixed posterior policy $\pi$, the penalty term in Proposition~\ref{prop:pac-bayes-bound} has the form:
\begin{equation}
    \frac{L_f^2 \, C_h^2}{4 N} \lambda + \frac{\kld{\pi}{\pi_0} + 2 \log(2/\delta)}{\lambda},
\end{equation}
and is minimized by the oracle choice
\begin{equation}
    \lambda_{\text{oracle}}(\pi) = \sqrt{\frac{4 N \left(\kld{\pi}{\pi_0} + 2 \log(2/\delta) \right)}{L_f^2 C_h^2}}.
\end{equation}
This expression is useful for interpreting the estimation--complexity trade-off of the bound. The optimal precision $\lambda_{\text{oracle}}(\pi)$ grows with the Monte Carlo budget $N$. Intuitively, as $N$ grows, the Monte Carlo error in $\tilde{I}_{\alpha}^{\timesit{S}}$ decreases, so the empirical robust EIG can be trusted more and the Gibbs policy concentrates more strongly around designs with high empirical robust EIG. 
At the same time, larger $N$ reduces the cost of the estimation penalty $\lambda L_f^2 C_h^2 / (4N)$, allowing a larger inverse temperature before the finite-sample correction becomes too large. Since the KL term is divided by $\lambda$, this also means that, for a fixed complexity $\kld{\pi}{\pi_0}$, larger Monte Carlo budgets permit larger deviations from the prior policy $\pi_0$ while maintaining a valid certificate.

Conversely, Proposition~\ref{prop:estimator-convergence} shows that $L_f$ and $C_h$ control the concentration of the Monte Carlo estimator. Larger values of either quantity negatively impact the estimator's reliability. To compensate, the optimal $\lambda$ shrinks, forcing the policy to stay closer to the prior to prevent overfitting to Monte Carlo noise. Overall, the choice of $\lambda$ should increase when the empirical estimator is reliable, and decrease when the estimator is more variable.

Because $\lambda_{\text{oracle}}$ depends on the learned policy $\pi$, itself learned under Monte Carlo samples, it cannot be computed directly. 
Thankfully, the PAC-Bayes bound in Proposition~\ref{prop:pac-bayes-bound} is valid for any value of $\lambda$ fixed independently of those samples. 
To tune $\lambda$, we can instead follow the standard PAC-Bayesian treatment described by \citet{alquier2024pacbayes}. 

To this end, consider a finite grid $\Lambda \subset (0, \infty)$, chosen before observing the Monte Carlo samples. 
Using the same argument as Proposition~\ref{prop:pac-bayes-bound}, applied separately to each fixed $\lambda\in\Lambda$ with confidence level $\delta/|\Lambda|$ and taking a union bound over the grid, we can prove that, with probability at least $1 - \delta$, for all $\pi \in \Pi$ and all $\lambda \in \Lambda$ we have:
\begin{equation}
    \mathbb{E}_{\pi} \big[I_{\alpha}^{\timesit{S}}(\xi) \big] \ge \mathbb{E}_{\pi} \big[ \tilde{I}_{\alpha}^{\timesit{S}}(\xi) \big] - \frac{\lambda L_{f}^{2} C_{h}^{2}}{4N} - \frac{\kld{\pi}{\pi_{0}} + 2 \log(2|\Lambda|/\delta)}{\lambda}.
\end{equation}

The choice of $\lambda$ is now clear: it may be selected together with the policy by maximizing the grid-based lower bound over $\lambda, \pi$:
\begin{equation}
    \pi^{\star}, \lambda^{\star} = {\arg\max}_{\pi \in \Pi, \lambda \in \Lambda} \,\, \mathbb{E}_{\pi} \big[ \tilde{I}_{\alpha}^{\timesit{S}}(\xi) \big] - \frac{\lambda L_{f}^{2} C_{h}^{2}}{4N} - \frac{\kld{\pi}{\pi_{0}} + 2 \log(2|\Lambda|/\delta)}{\lambda}.
\end{equation}
This is done by simple exhaustion: for each $\lambda \in \Lambda$, one computes the Gibbs policy $\pi_\lambda(\xi) \propto \pi_0(\xi) \exp \big\{ \lambda \, \tilde{I}_{\alpha}^{\timesit{S}}(\xi)\big\}$ and selects the value of $\lambda$ whose lower bound is largest.

In practice, evaluating the PAC-Bayes bound exactly may be infeasible when the constants $L_f$ and $C_h$ are unknown. To assess how sensitive the policy optimization problem in~\eqref{eq:design-policy-opt} is to the choice of $\lambda$, we can empirically study $\lambda$, using the regret of the learned policy $\pi_\lambda$ relative to a tractable Sibson's $\alpha$-MI benchmark in the linear regression setting.

To do so, we fix the inner sample size $M=16$ and perform a joint sweep over $\lambda \in \{10^{2}, \ldots, 10^{8}\}$ and the outer sample size $N \in \{64, 128, 256\}$ on the high-dimensional linear regression problem. For each pair $(N, \lambda)$, we learn a separate PAC-Bayes policy.
Table~\ref{tab:linreg-lambda-sweep} reports the mean relative regret over $4\,096$ design samples of the corresponding learned policy.
While individual entries are noisy, partly due to the stochastic mirror-descent optimizer, the results exhibit a stable low-regret plateau for $\lambda \in [10^{5}, 10^{7}]$, and this pattern is consistent across all tested values of $N$. Although this sweep is not intended as a formal tuning procedure for $\lambda$, its results are consistent with the evaluations in Section~\ref{sec:experiments}, where policies learned with $\lambda= 10^6$ achieve substantially lower regret than a naive optimizer.
\begin{table*}[t]
    \centering
    \scriptsize
    \caption{Sensitivity of PAC-Bayes policies to the choice of $\lambda$ on the linear regression problem. We sweep over $\lambda$ and $N$ and report the mean relative regret.}
    \label{tab:linreg-lambda-sweep}
    \begin{tabular}{l ccccccc}
        \hiderowcolors
        \toprule
        $N$ & $\lambda = 10^{2}$ & $\lambda = 10^{3}$ & $\lambda = 10^{4}$ & $\lambda = 10^{5}$ & $\lambda = 10^{6}$ & $\lambda = 10^{7}$ & $\lambda = 10^{8}$ \\
        \midrule
        \rowcolor{gray!30}
        64  & 0.062 & 0.020 & 0.008 & $\mathbf{0.011}$ & $\mathbf{0.026}$ & $\mathbf{0.023}$ & 0.019 \\
        128 & 0.081 & 0.037 & 0.037 & $\mathbf{0.024}$ & $\mathbf{0.006}$ & $\mathbf{0.013}$ & 0.023 \\
        \rowcolor{gray!30}
        256 & 0.094 & 0.019 & 0.020 & $\mathbf{0.021}$ & $\mathbf{0.019}$ & $\mathbf{0.022}$ & 0.018 \\
        \bottomrule
    \end{tabular}
\end{table*}

}

\section{Evaluation Details}
\label{app:evaluation-details}

All numerical evaluations reported in this paper require negligible compute and memory: every experiment runs on a standard CPU and completes within minutes.

\subsection{Linear Regression}
\label{app:linear-regression}

To illustrate the robust experimental design framework, we derive closed-form expressions for the worst-case distribution and robust expected information gain in the conjugate linear regression setting. This model admits tractable analytical results that provide intuition for the general case.

Consider the standard Bayesian linear regression model. Let $\theta = (\theta_{1}, \theta_{2})^{\top} \in \mathbb{R}^{2}$ denote the unknown parameter vector, where $\theta_{1} \in \mathbb{R}$ is the slope and $\theta_{2} \in \mathbb{R}$ is the offset. We endow $\theta$ with a Gaussian prior $p(\theta) = \mathrm{N}(\theta; \mu_{0}, \Sigma_{0})$, where $\mu_{0} \in \mathbb{R}^{2}$ is the prior mean and $\Sigma_{0} \in \mathbb{R}^{2 \times 2}$ is the prior covariance matrix. A batch design $\xi_{1:N} = (\xi_{1}, \ldots, \xi_{N}) \in \Xi$ specifies $N$ individual designs $\xi_{i} \in \mathbb{R}$. We construct the augmented design matrix $H(\xi_{1:N}) \in \mathbb{R}^{N \times 2}$ by appending a column of ones:
\begin{equation}
    H(\xi_{1:N}) = \begin{bmatrix} \xi_{1} & 1 \\ \vdots & \vdots \\ \xi_{N} & 1 \end{bmatrix}.
\end{equation}
Conditioned on the parameter $\theta$ and design $\xi_{1:N}$, the $N$ measurements are collected as a vector $x_{1:N} = (x_{1}, \ldots, x_{N})^{\top} \in \mathbb{R}^N$, where each $x_{i} \in \mathbb{R}$ is a scalar observation. Under the assumption of independence, the likelihood factorizes as:
\begin{equation}
    p(x_{1:N} \mid \theta, \xi_{1:N}) = \prod_{i=1}^{N} p(x_{i} \mid \theta, \xi_{i}) = \prod_{i=1}^{N} \mathrm{N}(x_{i}; \xi_{i} \, \theta_{1} + \theta_{2}, \sigma^{2}) = \mathrm{N}(x_{1:N}; H(\xi_{1:N}) \, \theta, \sigma^{2} I_{N}),
\end{equation}
where $\sigma^{2} > 0$ is the observation noise variance.

\textbf{Marginal and Posterior.} Under these conjugate assumptions, the marginal likelihood admits closed-form Gaussian expressions:
\begin{equation}
    p(x_{1:N} \mid \xi_{1:N}) = \mathrm{N}(x_{1:N}; H(\xi_{1:N}) \, \mu_{0}, H(\xi_{1:N}) \, \Sigma_{0} \, H(\xi_{1:N})^{\top} + \sigma^{2} I_{N}).
\end{equation}
Additionally, the posterior $p(\theta \mid x_{1:N}, \xi_{1:N}) = \mathrm{N}(\theta; \mu_{N}, \Sigma_{N})$ is given by the standard conjugate update formulas:
\begin{equation}
    \Sigma_{N} = \left[ \Sigma_{0}^{-1} + \frac{1}{\sigma^{2}} H(\xi_{1:N})^{\top} H(\xi_{1:N}) \right]^{-1}, \quad \mu_{N} = \Sigma_{N} \left[ \Sigma_{0}^{-1} \mu_{0} + \frac{1}{\sigma^{2}} H(\xi_{1:N})^{\top} x_{1:N} \right].
\end{equation}

\textbf{Worst-Case Distribution.} The worst-case joint distribution $q^{\star}(\theta, x_{1:N} \mid \xi_{1:N})$ is given by:
\begin{equation}
    q^{\star}(\theta, x_{1:N} \mid \xi_{1:N}) \propto \Big[ p(\theta) \, p_{\alpha}(x_{1:N} \mid \xi_{1:N}) \Big]^{1 - \alpha} \Big[ p(\theta, x_{1:N} \mid \xi_{1:N}) \Big]^{\alpha},
\end{equation}
where the tilted marginal $p_{\alpha}(x_{1:N} \mid \xi_{1:N})$ is defined as:
\begin{equation}
    p_{\alpha}(x_{1:N} \mid \xi_{1:N}) \propto \left[ \mathbb{E}_{p(\theta)} \Big[ p(x_{1:N} \mid \theta, \xi_{1:N})^{\alpha} \Big] \right]^{1/\alpha}.
\end{equation}
For this conjugate problem, the $\alpha$-powered likelihood is:
\begin{equation}
    p(x_{1:N} \mid \theta, \xi_{1:N})^{\alpha} = (2 \pi \sigma^{2})^{- \alpha N/2} \exp\left\{ -\frac{\alpha}{2\sigma^{2}} \|x_{1:N} - H(\xi_{1:N}) \, \theta\|^{2} \right\},
\end{equation}
where $\|x_{1:N} - H(\xi_{1:N}) \, \theta\|^{2} = \sum_{i=1}^{N} (x_{i} - \xi_{i} \theta_{1} - \theta_{2})^{2}$. To compute the expectation $\mathbb{E}_{p(\theta)}\big[p(x_{1:N} \mid \theta, \xi_{1:N})^{\alpha} \big]$, we expand the quadratic form and integrate over $\theta$. Completing the square in $\theta$ yields:
\begin{align}
    \mathbb{E}_{p(\theta)}\Big[ p(x_{1:N} \mid \theta, \xi_{1:N})^{\alpha} \Big] & = (2\pi\sigma^{2})^{-\alpha N/2} \int \exp\left\{ -\frac{\alpha}{2\sigma^{2}} \|x_{1:N} - H(\xi_{1:N}) \, \theta\|^{2} \right\} p(\theta) \, \dif \theta \\
    & = (2\pi\sigma^{2})^{-\alpha N/2} \frac{|\Psi_{\alpha}|^{1/2}}{|\Sigma_{0}|^{1/2}} \exp\left\{ -\frac{1}{2} \frac{\alpha}{\sigma^{2}} \|x_{1:N}\|^{2} - \frac{1}{2} \mu_{0}^{\top} \Sigma_{0}^{-1} \mu_{0} + \frac{1}{2} m_{\alpha}^{\top} \Psi_{\alpha}^{-1} m_{\alpha} \right\},
\end{align}
where we define the tilted covariance matrix and mean:
\begin{equation}
    \Psi_{\alpha} \coloneq \left[ \Sigma_{0}^{-1} + \frac{\alpha}{\sigma^{2}} H(\xi_{1:N})^{\top} H(\xi_{1:N}) \right]^{-1}, \quad m_{\alpha} \coloneq \Psi_{\alpha} \left[ \Sigma_{0}^{-1} \mu_{0} + \frac{\alpha}{\sigma^{2}} H(\xi_{1:N})^{\top} x_{1:N} \right].
\end{equation}
After simplification the tilted marginal takes the form:
\begin{equation}
    p_{\alpha}(x_{1:N} \mid \xi_{1:N}) \propto \exp\left\{ -\frac{1}{2} \big[x_{1:N} - H(\xi_{1:N}) \, \mu_{0} \big]^{\top} (\alpha \Lambda_{\alpha})^{-1} \big[x_{1:N} - H(\xi_{1:N}) \, \mu_{0} \big] \right\},
\end{equation}
where $\Lambda_{\alpha}$ is the $\alpha$-scaled predictive covariance for the batch:
\begin{equation}
    \Lambda_{\alpha}(\xi_{1:N}) := H(\xi_{1:N}) \, \Sigma_{0} \, H(\xi_{1:N})^{\top} + \frac{\sigma^{2}}{\alpha} \, I_{N}.
\end{equation}
Therefore, the tilted marginal is Gaussian:
\begin{equation}
    p_{\alpha}(x_{1:N} \mid \xi_{1:N}) = \mathrm{N}\big( x_{1:N}; H(\xi_{1:N}) \, \mu_{0}, \alpha \Lambda_{\alpha}(\xi_{1:N}) \big).
\end{equation}
Overall, this results in a worst-case joint distribution $q^{\star}(\theta, x_{1:N} \mid \xi_{1:N})$ that is also Gaussian:
\begin{equation}
    q^{\star}(\theta, x_{1:N} \mid \xi_{1:N}) = \mathrm{N}\left( \begin{bmatrix} \theta \\ x_{1:N} \end{bmatrix}; \mu_{q}^{\star}, \Sigma_{q}^{\star} \right),
\end{equation}
with mean and covariance:
\begin{equation}
    \mu_{q}^{\star} = \begin{bmatrix} \mu_{0} \\ H(\xi_{1:N}) \, \mu_{0} \end{bmatrix}, \quad \Sigma_{q}^{\star} = \begin{bmatrix} \alpha \Sigma_{0} + (1 - \alpha) \, \Sigma^{\star}_{\theta} & \alpha \Sigma_{0} \, H(\xi_{1:N})^{\top} \\ \alpha H(\xi_{1:N}) \, \Sigma_{0} & \alpha \Lambda_{\alpha}(\xi_{1:N}) \end{bmatrix},
\end{equation}
where the worst-case posterior $q^{\star}(\theta \mid x_{1:N}, \xi_{1:N}) = \mathrm{N}(\theta; \mu_{\theta}^{\star}, \Sigma_{\theta}^{\star})$ is:
\begin{equation}
    \Sigma_{\theta}^{\star} = \left[ \Sigma_{0}^{-1} + \frac{\alpha}{\sigma^{2}} H(\xi_{1:N})^{\top} H(\xi_{1:N}) \right]^{-1}, \quad \mu_{\theta}^{\star} = \Sigma_{\theta}^{\star} \left[ \Sigma_{0}^{-1}\mu_{0} + \frac{\alpha}{\sigma^{2}} H(\xi_{1:N})^{\top} x_{1:N} \right].
\end{equation}

\begin{figure}[t]
    \centering
    \pgfdeclareplotmark{custom-star}{
    \pgfmathsetmacro{\sz}{\pgfplotmarksize}
    \node[star,star point ratio=2.25,minimum size=5pt,
          inner sep=0pt,draw=black!70,solid,fill=black!90] {};
}

\begin{tikzpicture}
    \begin{groupplot}[
        group style={
            group size=3 by 1,
            horizontal sep=0.75cm,
            ylabels at=edge left,
        },
        width=5.0cm,
        height=5.0cm,
        xmin=-1.1, xmax=1.1,
        ymin=-1.1, ymax=1.1,
        xlabel={$\xi_1$},
        ylabel={$\xi_2$},
        view={0}{90},
        axis on top,
        xtick={-1, 1},
        ytick={-1, 1},
        axis x line*=bottom,
        axis y line*=left,
        xlabel style={yshift=10pt},
        ylabel style={yshift=-12pt},
        major tick length=2pt,
        tick align=outside,
    ]

    \nextgroupplot[title={$\alpha = 0.01$}]
        \addplot3[surf, shader=interp, colormap={grayfill}{gray(0cm)=(0.6); gray(1cm)=(1)}, mesh/rows=40, mesh/cols=40]
            table[x=xi_1, y=xi_2, z=mi_alpha_0.0100, col sep=comma] {figures/data/linreg_design.csv};
        \addplot3[contour prepared={labels=false, draw color=gray!60!black}]
            table {figures/data/linreg_contour_alpha_0.0100.table};
        \addplot[only marks, mark=custom-star] coordinates {(1, 1)};

    \nextgroupplot[title={$\alpha = 0.10$}, ylabel={}, yticklabels={}]
        \addplot3[surf, shader=interp, colormap={grayfill}{gray(0cm)=(0.6); gray(1cm)=(1)}, mesh/rows=40, mesh/cols=40]
            table[x=xi_1, y=xi_2, z=mi_alpha_0.1000, col sep=comma] {figures/data/linreg_design.csv};
        \addplot3[contour prepared={labels=false, draw color=gray!60!black}]
            table {figures/data/linreg_contour_alpha_0.1000.table};
        \addplot[only marks, mark=custom-star] coordinates {(1, 1)};

    \nextgroupplot[title={$\alpha \approx 1.00$}, ylabel={}, yticklabels={}]
        \addplot3[surf, shader=interp, colormap={grayfill}{gray(0cm)=(0.6); gray(1cm)=(1)}, mesh/rows=40, mesh/cols=40]
            table[x=xi_1, y=xi_2, z=mi_alpha_0.9950, col sep=comma] {figures/data/linreg_design.csv};
        \addplot3[contour prepared={labels=false, draw color=gray!60!black}]
            table {figures/data/linreg_contour_alpha_0.9950.table};
        \addplot[only marks, mark=custom-star] coordinates {(1, -1) (-1, 1)};

    \end{groupplot}
\end{tikzpicture}
    \vspace{-0.35cm}
    \caption{Robust expected information gain for a two-dimensional linear regression problem with a correlated prior, as a function of $\alpha \in (0, 1)$. Contour lines depict the objective landscape over $(\xi_{1}, \xi_{2})$, highlighting how the optimal designs ($\star$) shift with $\alpha$.}
    \label{fig:linreg-designs}
    \vspace{-0.4cm}
\end{figure}

\textbf{Sibson's $\alpha$-MI.} Finally, for this conjugate setting, Sibson's $\alpha$-mutual information admits a closed-form expression. Using the definition
\begin{equation}
    I_{\alpha}^{\timesit{S}}(\theta; x_{1:N})(\xi_{1:N}) = \frac{\alpha}{\alpha - 1} \log \int \left[ \int p(\theta) \, p(x_{1:N} \mid \xi_{1:N}, \theta)^{\alpha} \, \dif \theta \right]^{1/\alpha} \dif x_{1:N},
\end{equation}
which corresponds to the log-normalizer of $p_{\alpha}(x_{1:N} \mid \xi_{1:N})$:
\begin{align}
    I_{\alpha}^{\timesit{S}}(\theta; x_{1:N})(\xi_{1:N}) & = \frac{1}{2} \log \left| H(\xi_{1:N}) \, \Sigma_{0} \, H(\xi_{1:N})^{\top} + \frac{\sigma^{2}}{\alpha} I_{N} \right| + \mathrm{const} \\
    & = \frac{1}{2} \log \left| \frac{\alpha}{\sigma^{2}} H(\xi_{1:N}) \, \Sigma_{0} \, H(\xi_{1:N})^{\top} + I_{N} \right|.
\end{align}
As $\alpha \to 1$, the robust expected information gain recovers Shannon's mutual information, which is the standard result for Bayesian linear regression with $N$ measurements. As $\alpha \to 0$, the term $ \frac{\sigma^{2}}{\alpha} I_{N}$ dominates, and all designs become equally uninformative, consistent with Proposition~\ref{prop:uninformative}.

\textbf{R\'enyi Divergence.} The R\'enyi divergence between two Gaussian distributions is useful because it provides tractable computation of the realized robust information gain conditioned on a design and measured outcome. We derive the closed-form for orders $\alpha \in (0, 1)$. Consider two Gaussian distributions $q(\theta) = \mathrm{N}(\theta; \mu_{q}, \Sigma_{q})$ and $p(\theta) = \mathrm{N}(\theta; \mu_{p}, \Sigma_{p})$. The divergence is defined as:
\begin{equation}
    \mathbb{D}_{\alpha}[q \,\|\, p] = \frac{1}{\alpha - 1} \log \int q(\theta)^{\alpha} \, p(\theta)^{1 - \alpha} \, \dif \theta.
\end{equation}
Substituting the Gaussian density functions, the integrand becomes a product of exponential quadratic forms. The exponent corresponds to a new quadratic form defined by the geometric mixture of the parameters. Let $\Sigma_{\alpha}$ and $\mu_{\alpha}$ denote the covariance and mean of this geometric mixture, defined by the weighted sum of precision matrices:
\begin{equation}
    \Sigma_{\alpha} = \Big[ \alpha \Sigma_{q}^{-1} + (1 - \alpha) \Sigma_{p}^{-1} \Big]^{-1}, \quad \mu_{\alpha} = \Sigma_{\alpha} \Big[ \alpha \Sigma_{q}^{-1} \mu_{q} + (1 - \alpha) \, \Sigma_{p}^{-1} \mu_{p} \Big].
\end{equation}
Evaluating the Gaussian integral yields the log-integral term:
\begin{equation}
    \log \int q(\theta)^{\alpha} \, p(\theta)^{1-\alpha} \, \dif \theta = \frac{1}{2} \log \frac{|\Sigma_{\alpha}|}{|\Sigma_{q}|^{\alpha} |\Sigma_{p}|^{1-\alpha}} - \frac{1}{2} Z_{\alpha}(\mu_{q}, \mu_{p}),
\end{equation}
where $Z_{\alpha}$ is defined as follows:
\begin{equation}
    Z_{\alpha}(\mu_{q}, \mu_{p}) \coloneq \alpha \, \mu_{q}^{\top} \Sigma_{q}^{-1} \mu_{q} + (1 - \alpha) \, \mu_{p}^{\top} \Sigma_{p}^{-1} \mu_{p} - \mu_{\alpha}^{\top} \Sigma_{\alpha}^{-1} \mu_{\alpha}.
\end{equation}
Finally, scaling by $1 / (\alpha - 1)$ gives the divergence:
\begin{equation}
    \mathbb{D}_{\alpha}[q \,\|\, p] = \frac{1}{2 (1 - \alpha)} Z_{\alpha}(\mu_{q}, \mu_{p}) + \frac{1}{2(\alpha-1)} \log \frac{|\Sigma_{\alpha}|}{|\Sigma_{q}|^{\alpha} |\Sigma_{p}|^{1-\alpha}}.
\end{equation}

\subsection{A/B Testing}
\label{app:ab-testing}

To illustrate the robust experimental design framework in a setting with discrete measurements, we derive expressions for the worst-case distribution and robust expected information gain in the conjugate Beta-Binomial A/B testing setting.

Consider a standard A/B testing scenario where we wish to infer the conversion rates of two independent variants. Let $\theta = (\theta_{a}, \theta_{b})^\top \in [0, 1] \times [0, 1]$ denote the unknown parameter vector, where $\theta_{a}$ and $\theta_{b}$ are the conversion probabilities for groups A and B, respectively. We equip $\theta$ with a product of independent Beta priors:
\begin{equation}
    p(\theta) = \mathrm{Beta}(\theta_{a}; \delta_{a}, \gamma_{a}) \times \mathrm{Beta}(\theta_{b}; \delta_{b}, \gamma_{b}),
\end{equation}
where $\delta_{a}, \gamma_{a}, \delta_{b}, \gamma_{b} > 0$ are the prior hyperparameters. A design $\xi = (n_{a}, n_{b}) \in \Xi$ specifies the sample sizes allocated to each group, subject to a total budget constraint $n_{a} + n_{b} = N_x$. The measurements $x = (x_{a}, x_{b}) \in \{0, \dots, n_{a}\} \times \{0, \dots, n_{b}\}$ represent the number of conversions, or successes, in each group. Under the assumption of independence, the likelihood factorizes as a product of Binomial distributions:
\begin{align}
    p(x \mid \theta, \xi) & = \mathrm{Bin}(x_{a}; n_{a}, \theta_{a}) \times \mathrm{Bin}(x_{b}; n_{b}, \theta_{b}) \\
    & = \left[ \binom{n_{a}}{x_{a}} \, \theta_{a}^{\,x_{a}} (1 - \theta_{a})^{n_{a}-x_{a}} \right] \left[ \binom{n_{b}}{x_{b}} \theta_{b}^{\,x_{b}} (1 - \theta_{b})^{n_{b} - x_{b}} \right].
\end{align}

\textbf{Marginal and Posterior.} Under these assumptions, we can derive the marginal likelihood of the measurements $p(x \mid \xi)$ by integrating out the parameters $\theta$. This results in the Beta-Binomial distribution. For a single group $k \in \{a, b\}$, the marginal probability of observing $x_{k}$ successes in $n_{k}$ trials is:
\begin{align}
    p(x_{k} \mid n_{k}) & = \int_{0}^{1} \mathrm{Bin}(x_{k}; n_{k}, \theta_{k}) \, \mathrm{Beta}(\theta_{k}; \delta_{k}, \gamma_{k}) \, \dif \theta_{k} \\
    & = \frac{1}{B(\delta_{k}, \gamma_{k})} \binom{n_{k}}{x_{k}} \int_{0}^{1} \theta_{k}^{\,x_{k} + \delta_{k} - 1} (1-\theta_{k})^{n_{k} - x_{k} + \gamma_{k} - 1} \, \dif \theta_{k} \\
    & = \frac{1}{B(\delta_{k}, \gamma_{k})} \binom{n_{k}}{x_{k}} \, B(\delta_{k} + x_{k}, \gamma_{k} + n_{k} - x_{k}),
\end{align}
where $B(\cdot)$ is the Beta function and the joint marginal likelihood is simply the product $p(x \mid \xi) = p(x_{a} \mid n_{a}) \, p(x_{b} \mid n_{b})$.

Furthermore, under the assumptions of conjugacy, we can compute the Bayes posterior $p(\theta \mid x, \xi)$ in closed-form as factorized independent Beta distributions for each group. For $k \in \{a, b\}$:
\begin{align}
    p(\theta_{k} \mid x_{k}, n_{k}) & \propto \mathrm{Bin}(x_{k}; n_{k}, \theta_{k}) \, \mathrm{Beta}(\theta_{k}; \delta_{k}, \gamma_{k}) \\
    & \propto \theta_{k}^{\,x_{k}} (1 - \theta_{k})^{n_{k} - x_{k}} \, \theta_{k}^{\,\delta_{k} - 1} (1 - \theta_{k})^{\gamma_{k} - 1} \\
    & = \theta_{k}^{\,\delta_{k} + x_{k} - 1} \, (1 - \theta_{k})^{\gamma_{k} + n_{k} - x_{k} - 1} \\
    & = \mathrm{Beta}(\theta_{k}; \delta_{k} + x_{k}, \gamma_{k} + n_{k} - x_{k}).
\end{align}

\begin{figure}[t]
    \centering
    \begin{tikzpicture}
  \begin{groupplot}[
      group style={
        group size=3 by 1,
        horizontal sep=0.75cm,
        ylabels at=edge left
      },
      axis x line*=bottom,
      axis y line*=left,
      width=5.5cm,
      height=4cm,
      enlarge y limits={lower, value=0.03},
      xlabel={$\xi$},
      xlabel style={yshift=5pt},
      ylabel style={yshift=-3pt},
      title style={font=\footnotesize},
      major tick length=2pt,
      title style={yshift=-3pt},
      tick align=outside,
  ]
  
  \nextgroupplot[
    title={$\alpha = 0.01$},
    scaled y ticks=true,
    tick scale binop=\times,
    y tick scale label style={
        at={(axis description cs:0.02,1.02)},
        anchor=south west,
        xshift=0pt,
        yshift=-5pt,
        font=\scriptsize,
    },
  ]
      \addplot[ybar, bar shift=0pt, bar width=2.0pt, fill=gray!50, draw=none] table[
          x=n_A,
          y=mi_alpha_0.0100,
          col sep=comma
      ]{figures/data/abtesting_design.csv};
      \addplot[ybar, bar shift=0pt, bar width=2.0pt, fill=gray!150, draw=none] table[
          x=n_A, 
          y=max_alpha_0.0100, 
          col sep=comma
      ]{figures/data/abtesting_design.csv};
    
  \nextgroupplot[title={$\alpha=0.10$}]
      \addplot[ybar, bar shift=0pt, bar width=2pt, fill=gray!50, draw=none] table[
          x=n_A,
          y=mi_alpha_0.1000,
          col sep=comma
      ]{figures/data/abtesting_design.csv};
      \addplot[ybar, bar shift=0pt, bar width=2.0pt, fill=gray!150, draw=none] table[
          x=n_A, 
          y=max_alpha_0.1000, 
          col sep=comma
      ]{figures/data/abtesting_design.csv};
    
  \nextgroupplot[title={$\alpha \approx 1.00$}]
      \addplot[ybar, bar shift=0pt, bar width=2.0pt, fill=gray!50, draw=none] table[
          x=n_A,
          y=mi_alpha_1.0000,
          col sep=comma
      ]{figures/data/abtesting_design.csv};
      \addplot[ybar, bar shift=0pt, bar width=2.0pt, fill=gray!150, draw=none] table[
          x=n_A, 
          y=max_alpha_1.0000, 
          col sep=comma
      ]{figures/data/abtesting_design.csv};
    
  \end{groupplot}
\end{tikzpicture}
    \vspace{-0.35cm}
    \caption{Robust expected information gain as a function of $\alpha \in (0, 1)$ for an A/B testing problem with 25 participants. The optimal allocation, highlighted in dark gray, shifts as $\alpha$ varies.}
    \label{fig:abtesting-designs}
    \vspace{-0.4cm}
\end{figure}

\textbf{Worst-Case Distribution.} The worst-case joint distribution $q^{\star}(\theta, x \mid \xi)$ is given by:
\begin{equation}
    q^{\star}(\theta, x \mid \xi) \propto \Big[ p(\theta) \, p_{\alpha}(x \mid \xi) \Big]^{1 - \alpha} \Big[ p(\theta, x \mid \xi) \Big]^{\alpha},
\end{equation}
where the $\alpha$-tilted marginal $p_{\alpha}(x \mid \xi)$ is defined as:
\begin{equation}
    p_{\alpha}(x \mid \xi) \propto \left[ \mathbb{E}_{p(\theta)} \Big[ p(x \mid \theta, \xi)^{\alpha} \Big] \right]^{1/\alpha}.
\end{equation}
For this conjugate problem, we compute the expectation of the $\alpha$-powered likelihood. Due to independence, the expectation factorizes over the categories for $\theta_{a}$ and $\theta_{b}$:
\begin{equation}
    \mathbb{E}_{p(\theta)} \Big[ p(x \mid \theta, \xi)^{\alpha} \Big] = \prod_{k \in \{a, b\}} \mathbb{E}_{p(\theta_{k})} \Big[ p(x_{k} \mid \theta_{k}, n_{k})^{\alpha} \Big].
\end{equation}
Focusing on a single group with parameters $(\theta_{k}, n_{k}, x_{k})$ where $k \in \{a, b\}$:
\begin{equation}
    p(x_{k} \mid \theta_{k}, n_{k})^{\alpha} = \binom{n_{k}}{x_{k}}^{\alpha} \theta_{k}^{\,\alpha x_{k}} \, (1-\theta_{k})^{\alpha(n_{k} - x_{k})}.
\end{equation}
The expectation with respect to the Beta prior is:
\begin{align*}
    \mathbb{E}_{p(\theta_{k})} \Big[ p(x_{k} \mid \theta_{k}, n_{k})^{\alpha} \Big] & = \int_{0}^{1} \frac{\theta_{k}^{\,\delta_{k} - 1}(1 - \theta_{k})^{\gamma_{k} - 1}}{B(\delta_{k}, \gamma_{k})} \binom{n_{k}}{x_{k}}^{\alpha} \theta_{k}^{\,\alpha x_{k}} (1 - \theta_{k})^{\alpha(n_{k} - x_{k})} \dif \theta_{k} \\
    & = \frac{1}{B(\delta_{k}, \gamma_{k})} \, \binom{n_{k}}{x_{k}}^{\alpha} \int_{0}^1 \theta_{k}^{\,\delta_{k} + \alpha x_{k} - 1} (1-\theta_{k})^{\gamma_{k} + \alpha(n_{k} - x_{k}) - 1} \dif \theta_{k} \\
    & = \frac{1}{B(\delta_{k}, \gamma_{k})} \, \binom{n_{k}}{x_{k}}^{\alpha} B(\delta_{k} + \alpha \, x_{k}, \gamma_{k} + \alpha(n_{k} - x_{k})).
\end{align*}
Let us define the auxiliary term $Z_{\alpha, k}(x_{k}, n_{k})$ as this expectation:
\begin{equation}
    Z_{\alpha, k}(x_{k}, n_{k}) \coloneq \frac{1}{B(\delta_{k}, \gamma_{k})} \, \binom{n_{k}}{x_{k}}^{\alpha} B(\delta_{k} + \alpha \, x_{k}, \gamma_{k} + \alpha(n_{k} - x_{k})).
\end{equation}
The $\alpha$-tilted marginal $p_{\alpha}(x \mid \xi)$ then takes the form:
\begin{equation}
    p_{\alpha}(x \mid \xi) \propto \Big[Z_{\alpha, a}(x_{a}, n_{a}) Z_{\alpha, b}(x_{b}, n_{b}) \Big]^{1/\alpha}.
\end{equation}
To normalize this distribution, we compute the partition function:
\begin{equation}
    \mathcal{Z}_{\alpha}(\xi) = \left[ \sum_{x_{a} = 0}^{n_{a}} Z_{\alpha, a}(x_{a}, n_{a})^{1/\alpha} \right] \left[ \sum_{x_{b}=0}^{n_{b}} Z_{\alpha, b}(x_{b}, n_{b})^{1/\alpha} \right].
\end{equation}
Consequently, the worst-case joint distribution $q^{\star}(\theta, x \mid \xi)$ implies a worst-case posterior $q^{\star}(\theta \mid x, \xi)$ which is the $\alpha$-tilted posterior. For the Beta-Binomial model, this results in a Beta distribution with updated parameters:
\begin{equation}
    q^{\star}(\theta_{k} \mid x_{k}, n_{k}) = \mathrm{Beta}(\theta_{k}; \delta_{k} + \alpha \, x_{k}, \gamma_{k} + \alpha(n_{k} - x_{k})).
\end{equation}
This confirms that under model misspecification, $\alpha < 1$, the experimenter updates beliefs less aggressively than in the standard Bayesian setting.

\textbf{Sibson's $\alpha$-MI.} Sibson's $\alpha$-mutual information can be written as:
\begin{equation}
    I_{\alpha}^{\timesit{S}}(\theta; x)(\xi) = \frac{\alpha}{\alpha - 1} \log \sum_{x} \Big[ \mathbb{E}_{p(\theta)} \big[ p(x \mid \theta, \xi)^{\alpha} \big] \Big]^{1/\alpha}.
\end{equation}
Substituting our derived terms, this becomes $I_{\alpha}^{\timesit{S}}(\xi) = \alpha / (\alpha - 1) \log \mathcal{Z}_{\alpha}(\xi)$. Due to the independence of the groups, the information gain is additive $I_{\alpha}^{\timesit{S}}(\xi) = I_{\alpha}^{\timesit{S}}(n_a) + I_{\alpha}^{\timesit{S}}(n_b)$, where the contribution of a single group with allocation $n_{k}$ is:
\begin{equation}
    I_{\alpha}^{\timesit{S}}(n_{k}) = \frac{\alpha}{\alpha - 1} \log \left\{ \sum_{x_{k}=0}^{n_{k}} \left[ \binom{n_{k}}{x_{k}}^{\alpha} \frac{B(\delta_{k} + \alpha \, x_{k}, \gamma_{k} + \alpha(n_{k} - x_{k}))}{B(\delta_{k}, \gamma_{k})} \right]^{1/\alpha} \right\}.
\end{equation}
As $\alpha \to 1$, we recover the standard expected information gain for the Beta-Binomial model, which corresponds to Shannon's mutual information. As $\alpha \to 0$, the robust information gain approaches zero.

\textbf{R\'enyi Divergence.} Similar to the Gaussian case, the R\'enyi divergence between two Beta distributions admits a closed-form expression, which allows for tractable computation of the realized robust information gain given a design and outcome. Consider two Beta distributions $q(\theta) = \mathrm{Beta}(\theta; \delta_{q}, \gamma_{q})$ and $p(\theta) = \mathrm{Beta}(\theta; \delta_{p}, \gamma_{p})$. The divergence of order $\alpha$ is:
\begin{equation}
    \mathbb{D}_{\alpha}[q \,\|\, p] = \frac{1}{\alpha - 1} \log \int_{0}^{1} q(\theta)^{\alpha} \, p(\theta)^{1 - \alpha} \, \dif \theta.
\end{equation}
Substituting the Beta density functions, the integrand becomes a product of power functions. This results in an unnormalized Beta density with a geometric mixture of the original parameters. Let $\delta_{\alpha}$ and $\gamma_{\alpha}$ denote the effective parameters of this mixture, defined by the linear interpolation of the natural parameters:
\begin{equation}
    \delta_{\alpha} = \alpha(\delta_{q} - 1) + (1 - \alpha)(\delta_{p} - 1) + 1, \quad \gamma_{\alpha} = \alpha(\gamma_{q} - 1) + (1 - \alpha)(\gamma_{p} - 1) + 1.
\end{equation}
The integral can then be evaluated analytically using the Beta function $B(\cdot)$:
\begin{equation}
    \int_{0}^{1} q(\theta)^{\alpha} \, p(\theta)^{1-\alpha} \, \dif \theta = \frac{B(\delta_{\alpha}, \gamma_{\alpha})}{B(\delta_{q}, \gamma_{q})^{\alpha} \, B(\delta_{p}, \gamma_{p})^{1-\alpha}}.
\end{equation}
Taking the logarithm and scaling by $1 / (\alpha - 1)$ yields the closed-form divergence:
\begin{equation}
    \mathbb{D}_{\alpha}[q \,\|\, p] = \frac{1}{\alpha - 1} \log B(\delta_{\alpha}, \gamma_{\alpha}) - \frac{\alpha}{\alpha - 1} \log B(\delta_{q}, \gamma_{q}) - \frac{1 - \alpha}{\alpha - 1} \log B(\delta_{p}, \gamma_{p}).
\end{equation}

\section[Connection to Barlas]{Connection to \citet{barlas2025robust}}
\label{app:compare-barlas}

{\color{black}
In this section, we contrast Sibson's $\alpha$-mutual information with the power-likelihood version of the \emph{Gibbs expected information gain} (GEIG) presented by \citet{barlas2025robust}. For a power likelihood with exponent $\alpha>0$, the generalized likelihood $p(x \mid \theta,\xi)^\alpha$ is not a normalized density in $x$. Rather, it is the unnormalized likelihood factor used in the Gibbs posterior
\begin{equation}
    p_\alpha(\theta \mid x,\xi) = \frac{p(\theta) \, p(x\mid\theta,\xi)^\alpha}{\int p(\vartheta) \, p(x \mid \vartheta, \xi)^\alpha \dif \vartheta},
\end{equation}
and the denominator is the marginal generalized likelihood
\begin{equation}
    m_\alpha(x\mid\xi) \coloneq \int p(\theta) \, p(x \mid \theta, \xi)^\alpha \dif \theta.
\end{equation}
The marginal $m_\alpha(x \mid \xi)$ need not integrate to one over $x$ either. Under the \emph{pseudo-information} construction of \citet{barlas2025robust},\footnote{We follow the first version of \citet{barlas2025robust} (arXiv v1, November 2025). A later revision (v2, October 2026) normalizes this pseudo-expectation. Under that normalization, the GEIG coincides with Sibson's $\alpha$-mutual information in the linear regression example below.} the power-likelihood Gibbs expected information gain is
\begin{equation}
    I_\alpha^{\timesit{G}}(\xi) \coloneq \iint p(\theta) \, p(x \mid \theta, \xi)^\alpha
    \Big[
        \alpha \log p(x\mid\theta,\xi) - \log m_\alpha(x\mid\xi)
    \Big] \dif\theta\dif x.
\end{equation}
The powered likelihood thus appears in three places. It defines the unnormalized measure of the outer \emph{pseudo-expectation}, it enters the logarithm through $\alpha \log p(x\mid \theta,\xi)$, and it enters the marginal generalized likelihood $m_\alpha(x\mid\xi)$. The omission of the normalizing constants has direct consequences for the resulting information measure. To make this explicit, suppose we instead normalize the $\alpha$-powered likelihood over $x$ and define the $\alpha$-powered density
\begin{equation}
    f_\alpha(x \mid \theta,\xi) \coloneq \frac{p(x \mid \theta, \xi)^\alpha}{\int p(x \mid \theta, \xi)^\alpha \dif x} = \frac{p(x \mid \theta, \xi)^\alpha}{C_\alpha(\theta, \xi)}.
\end{equation}
Replacing $p(x \mid \theta, \xi)^\alpha$ by $f_\alpha(x \mid \theta, \xi)$ leaves the Gibbs posterior unchanged only if $C_\alpha(\theta, \xi)$ does not depend on $\theta$. Rewriting the marginal generalized likelihood in terms of $f_\alpha$ gives
\begin{equation}
    m_\alpha(x \mid \xi)
    = \int p(\theta) \, C_\alpha(\theta,\xi) \, f_\alpha(x \mid \theta,\xi) \dif \theta.
\end{equation}
In general, the normalizer $C_\alpha(\theta,\xi)$ is therefore not an irrelevant constant. If it depends on $\theta$, it changes the relative contribution of different parameter values to the marginal generalized likelihood. If it depends on $\xi$, it can also change the ranking of designs. Only in special cases where $C_\alpha(\theta, \xi)$ is independent of both $\theta$ and $\xi$ can it be dropped without affecting the posterior over $\theta$ or the design selected by maximizing the Gibbs expected information gain.

\paragraph{Linear Regression.}
We examine the linear regression case detailed in Appendix~\ref{app:linear-regression}. In this case, the likelihood of an outcome is defined as:
\begin{equation}
    p(x_{1:N}\mid\theta,\xi_{1:N}) = \mathrm{N}\big(x_{1:N}; H(\xi_{1:N}) \, \theta, \sigma^{2}I_{N}\big)
\end{equation}
\begin{equation}
    p(x_{1:N} \mid \theta, \xi_{1:N})^{\alpha} = C_{\alpha} \,\mathrm{N}\big(x_{1:N};H(\xi_{1:N}) \, \theta, \frac{\sigma^{2}}{\alpha}I_{N}\big),
    \qquad
    C_{\alpha} = (2\pi\sigma^{2})^{N(1-\alpha)/2}\alpha^{-N/2}.
\end{equation}
Hence the marginal is
\begin{equation}
    m_{\alpha}(x_{1:N}\mid\xi_{1:N}) = C_{\alpha} \, \mathrm{N}\big(x_{1:N}; H(\xi_{1:N}) \, \mu_{0}, H(\xi_{1:N}) \, \Sigma_{0}\, H(\xi_{1:N})^{\top} + \frac{\sigma^{2}}{\alpha}I_{N}\big),
\end{equation}
and
\begin{equation}
    I^{\timesit{G}}_{\alpha}(\xi_{1:N})
    = C_{\alpha} \, \frac{1}{2} \log \left|I_{N}+\frac{\alpha}{\sigma^{2}}H(\xi_{1:N}) \, \Sigma_{0} \, H(\xi_{1:N})^{\top}\right|
    = C_{\alpha} \, I_{\alpha}^{\timesit{S}}(\theta;x_{1:N})(\xi_{1:N}) .
\end{equation}
Thus, when $N$ and $\sigma^{2}$ are fixed, the power-likelihood generalized EIG from \citet{barlas2025robust} has the same maximizers as our closed-form Sibson objective for linear regression, but its value is scaled by the unnormalized generalized likelihood mass.

This is a special case where the normalizer $C_{\alpha}$ is independent of $\theta$ and $\xi$, so the omission of the normalizing constant does not affect the design selected by maximizing the acquisition function. However, this is not a general phenomenon, and in other settings the normalizer can have a significant impact on the resulting information measure and optimal design.

\paragraph{A/B Testing.}
For the A/B testing problem, detailed in Appendix~\ref{app:ab-testing}, the same simplification does not occur. The $\alpha$-powered Binomial likelihood has a normalizing constant that depends on the allocation, so the effect of using the unnormalized $\alpha$-powered likelihood cannot, in general, be reduced to a design-independent scale factor.

In this example, the likelihood of an outcome is given by the product of Binomial distributions:
\begin{equation}
    p(x \mid \theta, \xi) = \mathrm{Bin}(x_{a}; n_{a}, \theta_{a}) \times \mathrm{Bin}(x_{b}; n_{b}, \theta_{b}).
\end{equation}
For a single variant, define the normalizer of the $\alpha$-powered likelihood by
\begin{equation}
    C_{k,\alpha}(\theta_k, n_k) \coloneq \sum_{u=0}^{n_k} \mathrm{Bin}(u; n_k, \theta_k)^\alpha = \sum_{u=0}^{n_k} \binom{n_k}{u}^{\alpha} \theta_k^{\alpha u} \, (1-\theta_k)^{\alpha(n_k - u)}.
\end{equation}
For $\alpha=1$, this normalizer is equal to one. For $\alpha < 1$, it is generally a non-constant function of both $\theta_k$ and $n_k$. Thus the normalized $\alpha$-powered conditional density would be
\begin{equation}
    f_{k,\alpha}(x_k\mid\theta_k,n_k) \coloneq \frac{\mathrm{Bin}(x_k; n_k, \theta_k)^\alpha}{C_{k,\alpha}(\theta_k,n_k)}.
\end{equation}
The joint powered likelihood can therefore be written as $p(x\mid \theta, \xi)^\alpha = C_\alpha(\theta, \xi) \, f_\alpha(x \mid \theta,\xi)$, where
\begin{equation}
    C_\alpha(\theta, \xi) \coloneq \prod_{k\in\{a, b\}} C_{k, \alpha}(\theta_k, n_k),
    \qquad f_\alpha(x \mid \theta, \xi) \coloneq \prod_{k\in\{a, b\}} f_{k,\alpha}(x_k \mid \theta_k, n_k).
\end{equation}
Remember, the power-likelihood Gibbs EIG is written as
\begin{equation}
    I_\alpha^{\timesit{G}}(\xi)
    =
    \sum_x \int
    p(\theta) \,
    C_\alpha(\theta, \xi) \,
    f_\alpha(x \mid \theta, \xi)
    \Big[
        \log C_\alpha(\theta, \xi)
        + \log f_\alpha(x \mid \theta, \xi)
        - \log m_\alpha(x \mid \xi)
    \Big]
    \dif \theta.
\end{equation}
By contrast, Sibson's $\alpha$-mutual information for the same problem is
\begin{equation}
    I_\alpha^{\timesit{S}}(\theta; x)(n_a,n_b) = \frac{\alpha}{\alpha-1}
    \log \sum_{x_a=0}^{n_a} \sum_{x_b=0}^{n_b} m_\alpha(x\mid\xi)^{1/\alpha}.
\end{equation}
For $\alpha < 1$, these two objectives are not equivalent and can induce different optimal designs. We demonstrate this effect empirically across a geometric range of $\alpha$ values in Table~\ref{tab:sibson-vs-geig}.
}

\end{document}